\documentclass[sn-mathphys,Numbered,pdflatex]{sn-jnl}% Math and Physical Sciences Reference Style

\usepackage{lipsum}
\usepackage{amsmath}
\usepackage{dsfont,amssymb,amsfonts}
\usepackage{cancel}
\usepackage{float}
\usepackage{pgf,tikz, pgfplots}
\usepackage{hyperref}
\pgfplotsset{compat=1.16}
\usetikzlibrary{arrows, automata, graphs}
\usepackage{subcaption}
\usepackage[T1]{fontenc}

\usepackage{graphicx}%
\usepackage{multirow}%
\usepackage{amsmath,amssymb,amsfonts}%
\usepackage{amsthm}%
\usepackage{mathrsfs}%
\usepackage[title]{appendix}%
\usepackage{xcolor}%
\usepackage{textcomp}%
\usepackage{manyfoot}%
\usepackage{booktabs}%
\usepackage[ruled, linesnumbered]{algorithm2e}%
\usepackage{listings}%
\newcommand{\revision}[1]{{#1}}
\newcommand{\rerevision}[1]{{#1}}
\newcommand{\minrevision}[1]{{#1}}

\let\tempCaption=\caption
\renewcommand{\caption}[1]{\itshape \tempCaption{#1}}

\newcommand{\set}[1]{\left\{#1\right\}}
\newcommand{\nset}[2]{\set{#1,\dots,#2}}
\newcommand{\pa}[1]{\left(#1\right)}
\newcommand{\ang}[1]{\left<#1\right>}
\newcommand{\bra}[1]{\left[#1\right]}
\newcommand{\abs}[1]{\left|#1\right|}
\newcommand{\norm}[1]{\left\|#1\right\|}
\newcommand{\normop}[1]{{\left\vert\kern-0.25ex\left\vert\kern-0.25ex\left\vert #1
			\right\vert\kern-0.25ex\right\vert\kern-0.25ex\right\vert}}

\newcommand{\mat}[1]{\begin{matrix}#1\end{matrix}}

\newcommand{\bmat}[1]{\bra{\mat{#1}}}

\newcommand{\cas}[1]{\begin{cases}#1\end{cases}}

\newcommand{\der}[2]{\frac{\partial #1}{\partial #2}}

\newcommand{\restr}[2]{{
\left.\kern-\nulldelimiterspace
#1
\vphantom{\big|}
\right|_{#2}
}}

\newcommand{\fun}[5]{  % dessine la fonction #1 : #2 -> #3
	\begin{array}{ccccl} %							#4 -> #5
		#1 & : & #2 & \longrightarrow    & #3 \\
		   &   & #4 & \longmapsto & #5 \\
	\end{array}}

\usepackage{import}
\usepackage{xifthen}
\usepackage{pdfpages}
\usepackage{transparent}

\DeclareMathOperator{\diag}{diag}

\DeclareMathOperator{\Ima}{Im}

\DeclareMathOperator*{\argmax}{argmax}

\DeclareMathOperator{\Span}{Span}

\DeclareMathOperator{\Ker}{Ker}

\def\LC{\mathring{\nabla}}

\def\E{\mathds{E}}

\def\M{\mathcal{M}}

\def\R{\mathbb{R}}

\def\X{\mathcal{X}}
\def\Y{\mathcal{Y}}

\def\1{\mathds{1}}

\theoremstyle{plain}%
\newtheorem{thm}{Theorem}[section]% meant for sectionwise numbers
\newtheorem{prop}[thm]{Proposition}% 
\newtheorem{lem}[thm]{Lemma}
\newtheorem{cor}[thm]{Corollary}

\theoremstyle{remark}%
\newtheorem{rem}{Remark}%

\theoremstyle{definition}%
\newtheorem{definition}{Definition}%

\begin{document}

\title{\revision{Adversarial attacks on neural networks through canonical Riemannian foliations}}

\author*[1]{\fnm{Eliot} \sur{Tron}}\email{eliot.tron@enac.fr}

\author[1,2]{\fnm{Nicolas} \sur{Couellan}}\email{nicolas.couellan@recherche.enac.fr}

\author[1,2]{\fnm{St\'{e}phane} \sur{Puechmorel}}\email{stephane.puechmorel@enac.fr}

\affil*[1]{\orgname{Ecole Nationale de l'Aviation Civile}, \orgaddress{\street{7 Avenue Edouard Belin}, \city{Toulouse}, \postcode{31400}, \country{France}}}

\affil[2]{\orgname{Institut de Mathématiques de Toulouse}, \orgdiv{UMR 5219}, Universit\'{e} de Toulouse, CNRS, UPS, \orgaddress{\street{118 route de Narbonne}, \city{Toulouse}, \postcode{F-31062} Cedex 9, \country{France}}}

\abstract{% 
				Deep learning models are known to be vulnerable to adversarial attacks. Adversarial learning is therefore becoming a crucial task. We propose a new vision on neural network robustness using Riemannian geometry and foliation theory. The idea is illustrated by creating a new adversarial attack that takes into account the curvature of the data space. This new adversarial attack, called the \emph{two-step spectral attack}, is a piece-wise linear approximation of a geodesic in the data space. The data space is treated as a (pseudo) Riemannian manifold equipped with the pullback of the Fisher Information Metric (FIM) of the neural network. In most cases, this metric is only semi-definite and its kernel becomes a central object to study. A canonical foliation is derived from this kernel. The curvature of transverse leaves gives the appropriate correction to get a two-step approximation of the geodesic and hence a new efficient adversarial attack. The method is first illustrated on a 2D toy example in order to visualize the neural network foliation and the corresponding attacks. \revision{Next, we report numerical results on the \texttt{MNIST} and \texttt{CIFAR10} datasets with the proposed technique and state of the art attacks presented in \cite{zhaoAdversarialAttackDetection2019} (OSSA) and \cite{croce_reliable_2020} (AutoAttack).} The results show that the proposed attack is more efficient at all levels of available budget for the attack (norm of the attack), confirming that the curvature of the transverse neural network FIM foliation plays an important role in the robustness of neural networks. \rerevision{The main objective and interest of this study is to provide a mathematical understanding of the geometrical issues at play in the data space when constructing efficient attacks on neural networks.}
}

\keywords{
Neural Networks, Robustness, Fisher Information Metric, Information geometry, Adversarial attacks.
}

\maketitle

\section{Introduction}
\label{sec:introduction}

Lately there has been a growing interest in the analysis of neural network robustness and the sensitivity of such models to input perturbations (\cite{Fawzi2018,Shaham2018,Wong2018,raghunathanCertifiedDefensesAdversarial2020}). Most of these investigations have highlighted their weakness to handle adversarial attacks (\cite{Szegedy2014}) and have proposed some means to increase their robustness. Adversarial attacks are real threats that could slow down or eventually stop the development of neural network models or their applications in contexts where robustness guarantees are needed. For example, in the specific case of aviation safety, immunization of critical systems to adversarial attacks should not only be guaranteed but also certified. Therefore, addressing the robustness of future on-board or air traffic control automated systems based on such models is a main concern.

Adversarial attacks are designed to fool classification models by introducing perturbations in the input data. These perturbations remain small and in the case of images for example, may be undetectable to the human eye. So far, most of the research effort has focused on designing such attacks in order to augment the training dataset with the constructed adversarial samples and expecting that training will be more robust. Among these methods, one can refer to the Fast Gradient Sign methods (\cite{goodfellow2015}), robust optimization methods (\cite{madry2018}), DeepFool (\cite{deepfool2016}), and others (\cite{fawzi2017}). There are major drawbacks with these approaches. Most attacks are data dependent and provide no guarantees that all relevant attacks have been considered and added to the training set. Furthermore, training in such manner does not acquire a global knowledge about the weakness of the model towards adversarial threats. It will only gain robustness for attacks that have been added to the dataset. However, crafting adversarial attacks by exploiting the properties of neural network learning is useful to understand the principles at play in the robustness or sensitivity of neural architectures. 

Many authors consider neural network attacks and robustness properties in a Euclidean input space. Yet, it is commonly admitted that to learn from high dimensional data, data must lie in a low dimensional manifold (\cite{fefferman2016}). Such manifold has in general non-zero curvature and Riemannian geometry should therefore be a more appropriate setting to analyze distances and sensitivity from an attack point of view. Furthermore, to analyze neural network model separation capabilities and its robustness, it is critical to understand not only the topology of the decision boundaries in the input space but also the topology of iso-information regions induced by the neural network. Again, there is no reason to believe that these sub-manifolds have zero curvature in general. The Fisher information metric (FIM) is a valid metric for such purpose. Indeed, the network output is seen as a discrete probability that lies on a statistical manifold. The FIM may then be used as a Riemannian metric at the output and the pullback metric of the Fisher information as a metric for the input manifold (\cite{zhaoAdversarialAttackDetection2019}). The importance of the FIM in the context of deep neural networks has already been pointed out by several authors. 
In \cite{karakida19a}, it is shown that the FIM defines the landscape of the parameter space and the maximum FIM eigenvalue defines an approximation of the appropriate learning rate for gradient methods. The FIM with respect to data (also called local data matrix) instead of network parameters has also been investigated from a geometric perspective in \cite{grementieriModelcentricDataManifold2021}. The authors have shown that training data are organized on sub-manifolds (leaves) of a foliation of the data domain. A few authors have also tried to exploit this geometric knowledge to construct adversarial attacks or get some form of immunization from them. In \cite{zhaoAdversarialAttackDetection2019}, the direction of eigenvector corresponding to the maximum eigenvalue of the pullback FIM metric is used as a direction of attack where as in \cite{shenDefendingAdversarialAttacks2019}, similar developments are proposed to robustify the model by regularizing the neural network loss function by the trace of the FIM. 
\rerevision{
In~\cite{yan2024enhance,picot2022adversarial}, the authors directly use the geodesic distance associated with the FIM as a regularisation term during the training of the neural network. They show that exploiting the geometry of the FIM makes the network more robust to adversarial attacks. In~\cite{picot2022adversarial}, they also suggest that some notion of curvature of the statistical manifold is linked to the robustness of the network to adversarial attacks, as we will investigate in this article by looking at the foliated input space.
}
More specifically, in \cite{zhaoAdversarialAttackDetection2019}, the authors have shown that the one step spectral attack (OSSA) they proposed is more efficient than the Fast Gradient Sign and the One Step Target Class methods \cite{kurakin2016}. \rerevision{This shows that the choice of the FIM metric as a measure of sensitivity is relevant and supports our decision to also use it in our study. However, unlike this previous work, we will investigate its geometrical properties directly in the data space where attacks are constructed.} \rerevision{Note that other kinds of geometry have also been investigated, such as Wasserstein geometry in~\cite{zhu2023interpolation} where the authors show that Wasserstein barycenters make good candidates for robust data augmentation.
}

In this work, we build on the work of \cite{zhaoAdversarialAttackDetection2019} and exploit further the geometrical properties of the foliation of the pullback metric of the neural network FIM. More specifically, we show that the curvature of the leaf of the transverse foliation can be utilized to construct a two-step attack procedure referred as \textit{two-step} attack. Given a budget of attack, meaning the Euclidean norm of the attack vector, we first move in the direction of the eigenvector corresponding to the maximum eigenvalue of the FIM as proposed in \cite{zhaoAdversarialAttackDetection2019} and then make another move that takes into account the curvature. The two steps could be seen as a discretized move along a geodesic curve. The interest of this procedure is not only to prove that, for a given budget of attack, it is possible to construct worse attacks than those proposed in \cite{zhaoAdversarialAttackDetection2019} but also and more importantly to emphasize the role of curvature in the sensitivity of neural network models. Mathematically, this translates into the expression of the quadratic form approaching the Kullback-Leibler divergence between the network output probability distribution at the origin of the attack and the probability distribution at the point reached by the attack. Indeed, we show that this expression makes explicit use of the Riemannian curvature tensor of the foliation in the input space. To better grasp and visualize the effect of curvature on the attacks, we provide details on the calculation of the neural network FIM and the corresponding attacks on a \texttt{XOR} toy problem. The small dimension of the problem allows also 2D illustrations of the phenomenon. To demonstrate experimentally the benefit of exploiting the foliation curvature, experiments are conducted first with the \texttt{XOR} toy problem and next with the \texttt{MNIST} (\cite{lecun1998mnist}) and \texttt{CIFAR10} (\cite{cifar-10}) dataset. \revision{A measure of attack efficiency is defined using the fooling rate that computes the percentage of predictions that are changing class on random data points. The fooling rate is then used to compare the \textit{two-step} attack to OSSA, and the state of the art technique AutoAttack \cite{croce_reliable_2020}.} The result show a substantial increase for the fooling rates when curvature is utilized.  \\
The contribution of this study is two-fold: \rerevision{First and mainly, it proposes a Riemannian geometry framework for analyzing the robustness of neural networks. Gaining mathematical understanding on the efficiency of attacks should give future perspectives on strong defense mechanisms for neural networks. Next, since we use the construction of an attack as a methodology to understand the role of the curvature of the transverse leaves in the efficiency of an attack, we also provide as a by-product a two-step procedure for crafting adversarial attacks. These attacks have proven to be more effective than state of the art attacks on small neural architectures. However, there are limitations in the generalization of our construction technique on larger instances as it involves the computation of the whole Jacobian of the neural network. This computation may turn out to be untractable for very large networks.}
\\
The paper is organized as follows. Section~\ref{sec:problem_statement} details the general mathematical framework of this study and defines precisely the adversarial attacks that are considered. In the first part of Section~\ref{sec:local_method}, the construction of the local attack by \cite{zhaoAdversarialAttackDetection2019} is recalled. \revision{Next, in Section~\ref{sec:dogleg_attack}, the main developments of this research are detailed. They consist in extending the local attack using the geometry of the problem and constructing the so-called two-step attack.} Section~\ref{sec:example} is dedicated to the illustration of the method through a simple play-test problem. It details more precisely the required calculations on a simple low dimensional problem and provides illustrations of the resulting foliation at the heart of the technique. Section~\ref{sec:num_res} provides the numerical results and Section~\ref{sec:conclusion} concludes the article. Appendices~\ref{app:intro_geo} and \ref{app:FIM} recall some important concepts of Riemannian geometry that are used throughout the article. \revision{Appendix~\ref{app:proofs} contains the proof of some stated theoretical result.}

\section{Problem statement}
\label{sec:problem_statement}

\subsection{Setup}
\label{ssec:setup}

In this paper, we are studying the behavior of a neural network $N:\X\to\M$. Its output can be considered as a parameterized probability density function $p_\theta\pa{y\mid x} \in \M$ where $\M$ is a manifold of such probability density functions, $x\in\X$ is the input,  $y\in \Y$ is the targeted label and $\theta\in \Theta$ is the parameter of the model (for instance the weights and biases in a perceptron). The geometric study of such probability distributions is part of \emph{Information Geometry} \cite{nielsenElementaryIntroductionInformation2020}.

In this case, $\M$ and $\X$ are two (pseudo) Riemannian Manifolds when equipped with the \emph{Fisher Information Metric} (FIM). Fisher Information is originally a way to measure the variance of a distribution along a parameter. It was used as a Riemannian metric by Jeffrey and Rao (1945-1946), then by Amari, which gave birth to \emph{Information Geometry} \cite{amari2016information, nielsenElementaryIntroductionInformation2020}. Measuring how the distribution of the predicted labels changes with input perturbations falls exactly in our usecase. This explains why the FIM is a good candidate to be $\X$ and $\M$'s metric.

\begin{definition}[Fisher Information Metric]
The Fisher Information Metric (FIM) on the manifold $\X$ at the point $x$ is defined by the following positive semi-definite symmetric matrix:
\[g_{i j}(x) = \E_{y \mid x,\theta}\bra{\partial_{x_i} {\ln p(y\mid x,\theta)} \partial_{x_j} {\ln p(y\mid x,\theta)}}.\]
\end{definition}

\begin{rem}
Note that this definition is not the usual definition of the Fisher Information. Indeed, the one defined by Fisher, used in the work of Amari and many others, differentiate with respect to the parameter $\theta$ where we differentiate with respect to the input $x$. The authors in~\cite{grementieriModelcentricDataManifold2021} call this new metric the \emph{local data matrix} to avoid confusion. 
\end{rem}

It is important to notice that for the probability distribution given by a neural network, the Fisher Information is only a pseudo-Riemmannian metric, meaning that the matrix $G=\pa{g_{i j}}_{i,j}$ is not full rank. This is mainly due to the fact that $\dim \X \gg \dim \M$ in most cases of neural networks classifiers.

In this pseudo-Riemmannian geometric setting, one can define the kernel of the metric at each point. This subset of the tangent space $T\X$ is called a distribution and it is integrable under some assumptions, meaning that it gives rise to a \emph{foliation} i.e. a partition of $\X$ into submanifolds called \emph{leaves}. With this canonical foliation, a dual one can be defined using the orthogonal\footnote{Orthogonal with respect to the ambiant Euclidean metric.} complement of the kernel at each point. It is called the \emph{transverse} foliation.  At each point, moving along the FIM kernel leaf will not modify the output of the neural network. The properties of the transverse foliation will thus be of great interest in the setting of adversarial attacks. The reader may find more details on these geometric objects in the appendix~\ref{app:intro_geo} and~\ref{app:FIM}.

\subsection{Adversarial attacks}

One primary goal of this article is to craft the best \emph{adversarial attack} possible, or at least approach it. An adversarial attack aims at disturbing the output of the neural network by adding noise to the original input, with the intention of changing the predicted class and thus fooling the network.
The FIM with respect to the input is a good measure of dissimilarity between outputs, given a displacement of the input. In order to change the predicted class after the attack one would want to maximize the dissimilarity between the predicted distribution of probability before and after the attack. In other words, one would like to maximize the FIM's geodesic distance\footnote{see Definition~\ref{def:geo_dist} in~\ref{app:intro_geo}} $d(x_o,x_a)$ between the input point $x_o$ and the point $x_a$ after the attack.

\begin{prop}
The geodesic distance can be expressed with the Riemannian norm and the logarithm map\footnote{see Definition~\ref{def:log_map} in~\ref{app:intro_geo}}:
\[
d(x_o,x_a)^2 = \norm{\log_{x_o} x_a}_{\X}^2 = g_{x_o}\pa{\log_{x_o} x_a, \log_{x_o} x_a}
.\]
Otherwise said, if $v=\log_{x_o} x_a$ is the initial velocity of the geodesic between the two points,
\[
d(x_o,\exp_{x_o}(v))^2 = \norm{v}_{\X}^2 
.\]
\end{prop}

The optimal solution will thus be the $\hat{x}_a$ maximizing this quantity, with some constraints. Indeed, one of the characteristics of a good adversarial attack is to be as undetectable as possible by usual measurements on the input. A usual measure for that is the Euclidean distance between $x_o$ and $x_a$. The attacked point will thus be constrained to the Euclidean ball of given (small) radius and centered at $x_o$. The optimal attacked point is the result of a geodesic move from $x_o$ with initial velocity $v\in T_{x_o} \X$.

\begin{definition}[$\varepsilon$-Adversarial Attack Problem]
The optimal geodesic attack with a Euclidean budget of $\varepsilon>0$ and initial velocity $v$ verifies: 

\begin{equation}\label{eq:AdvAttack}
				\max_{v\in T_{x_o}\X} \norm{v}_\X^2 
				\text{ subject to }
				\norm{\exp_{x_o}(v) - x_o}_2^2 \leq \varepsilon^2       
				 \tag*{($\varepsilon$-AAP)} % TODO: |v|=max |tilde v| ?
\end{equation}
\end{definition}
\begin{rem}
				Some other authors \cite{fawzi2018adversarial} are looking for the smallest attack in Euclidean norm that changes the predicted class. We are taking the problem in reverse and looking for an imperceptible attack with respect to a sensor with an activation level of $\varepsilon$ (e.g. the human eye, the human ear, camera or some other devices) with no guarantees that there exists a class changing attack. However, in real life context where dimension is large, it has been shown that such class changing attacks exist for small $\varepsilon$ with high probability \cite{gilmer2018adversarial}.
\end{rem}

A usual optimization algorithm could solve this problem, but the amount of calculus would be tremendous since the geodesic requires to solve an ODE with boundary conditions each time. This greedy solution is thus not feasible in practice. To get close to the optimal solution, the geodesic can be approximated by Euclidean steps.

\section{A local method}\label{sec:local_method}

The authors in \cite{zhaoAdversarialAttackDetection2019} have proposed a method to approximate the solution of a near-\ref{eq:AdvAttack} problem. \revision{Their formulation of the Adversarial Attack Problem is not exactly the same since the criteria they are maximizing is not the geodesic distance, but the Kullback-Leibler divergence.} The matter with (KL) divergences is that they are not distances, because often not symmetric. Nevertheless, if $G$ is the matrix associated to the FIM, a second order Taylor approximation gives us that:
\begin{align*}
& D_{KL}\pa{p_\theta\pa{y\mid x}\parallel p_\theta\pa{y\mid x+ v}} = \frac{1}{2} v^T G_x v + o\pa{\norm{v}^2} \\
& = \frac{1}{2} g_x\pa{v,v} + o\pa{\norm{v}^2} = \frac{1}{2} \norm{v}_\X^2 + o\pa{\norm{v}^2}.
\end{align*}
Then, the Euclidean constraint is taken directly on $v$ meaning that $x_o$ is moved by a Euclidean step in the direction of $v$. 
The problem thus solved in \cite{zhaoAdversarialAttackDetection2019} is a linear approximation of \ref{eq:AdvAttack}. Doing so amounts to approaching the geodesic by its initial velocity.
In other words, the first order approximation $\exp_{x_o}(v) \approx x_o + v$ is used for this section.

\revision{We reformulate and prove below the main result of~\cite{zhaoAdversarialAttackDetection2019} in the context of this study.}

\begin{prop}\label{prop:ossa_vector}
    The vector $\hat{v}$ maximizing the quadratic form $v \mapsto v^T G_x v$ is an eigenvector of the FIM $G_x$ corresponding to the largest eigenvalue. Its re-normalization by $\frac{\varepsilon^2}{\norm{\hat{v}}_2^2}$ gives an approximated solution to \ref{eq:AdvAttack}. This method is illustrated by~\autoref{fig:1step_attack}.
\end{prop}

\begin{figure}[ht]
    \centering
				\def\svgwidth{0.4\columnwidth}
			%% Creator: Inkscape inkscape 0.92.5, www.inkscape.org
%% Accompanies image file 'one-step.pdf' (pdf, eps, ps)
\begingroup%
  \makeatletter%
  \providecommand\color[2][]{%
    \errmessage{(Inkscape) Color is used for the text in Inkscape, but the package 'color.sty' is not loaded}%
    \renewcommand\color[2][]{}%
  }%
  \providecommand\transparent[1]{%
    \errmessage{(Inkscape) Transparency is used (non-zero) for the text in Inkscape, but the package 'transparent.sty' is not loaded}%
    \renewcommand\transparent[1]{}%
  }%
  \providecommand\rotatebox[2]{#2}%
  \newcommand*\fsize{\dimexpr\f@size pt\relax}%
  \newcommand*\lineheight[1]{\fontsize{\fsize}{#1\fsize}\selectfont}%
  \ifx\svgwidth\undefined%
    \setlength{\unitlength}{484.61043711bp}%
    \ifx\svgscale\undefined%
      \relax%
    \else%
      \setlength{\unitlength}{\unitlength * \real{\svgscale}}%
    \fi%
  \else%
    \setlength{\unitlength}{\svgwidth}%
  \fi%
  \global\let\svgwidth\undefined%
  \global\let\svgscale\undefined%
  \makeatother%
  \begin{picture}(1,0.7353281)%
    \lineheight{1}%
    \setlength\tabcolsep{0pt}%
    \put(0,0){\includegraphics[width=\unitlength,page=1]{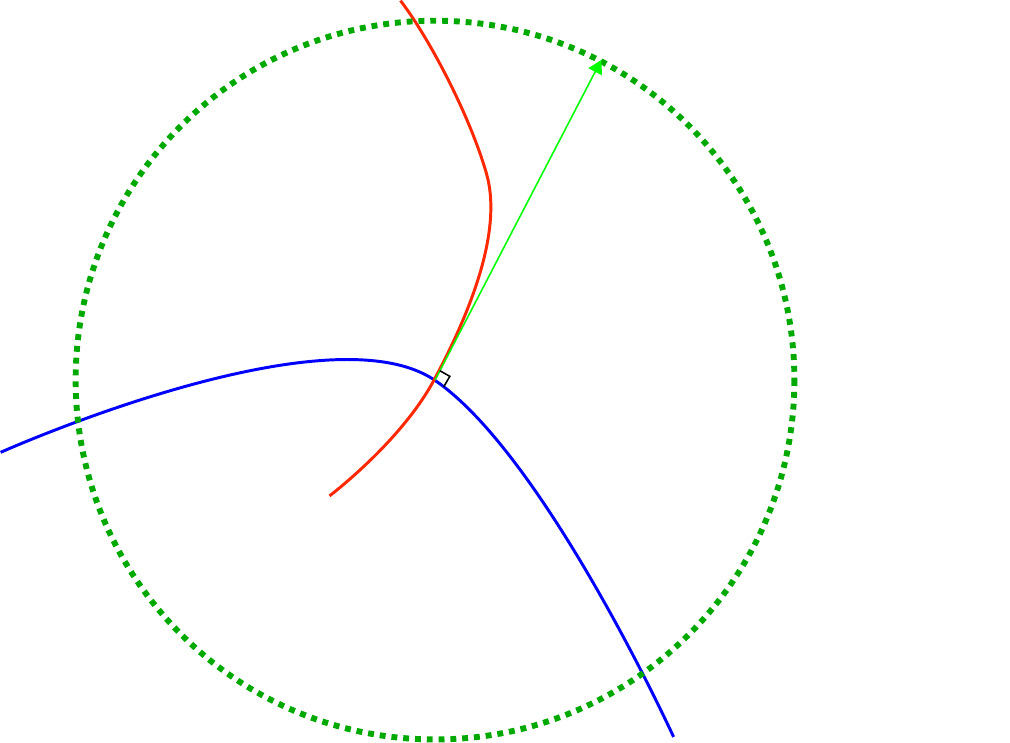}}%
    \put(0.50865986,0.46195035){\color[rgb]{0,1,0}\transparent{0.46666667}\makebox(0,0)[lt]{\lineheight{1.25}\smash{\begin{tabular}[t]{l}$v$\end{tabular}}}}%
    \put(0.18106188,0.12229883){\color[rgb]{0,0.6627451,0}\makebox(0,0)[lt]{\lineheight{1.25}\smash{\begin{tabular}[t]{l}$\varepsilon$\end{tabular}}}}%
    \put(0.5887719,0.69177964){\color[rgb]{0,0,0}\makebox(0,0)[lt]{\lineheight{1.25}\smash{\begin{tabular}[t]{l}$x_a$\end{tabular}}}}%
    \put(0.45000324,0.35498444){\color[rgb]{0,0,0}\makebox(0,0)[lt]{\lineheight{1.25}\smash{\begin{tabular}[t]{l}$x_o$\end{tabular}}}}%
  \end{picture}%
\endgroup%

    \caption{The one-step spectral attack in action. The circle represents the Euclidean budget. The blue curve represents the leaf of the kernel foliation and the red curve represents the transverse foliation.}
    \label{fig:1step_attack}
\end{figure}

\revision{The proof of Proposition~\ref{prop:ossa_vector} can be found in \autoref{app:proofs}.\\}

\begin{rem}
				The KKT optimality conditions lead to two valid and yet very different solutions: $\hat{v}$ and $-\hat{v}$. To choose between the two, we select the solution decreasing the probability of the original class the most, meaning that we select the attack vector $v\in\set{\hat{v},-\hat{v}}$ satisfying:
    \begin{equation}
                 p_\theta\pa{y_{x_o} \mid x_o - v} > p_\theta\pa{y_{x_o} \mid x_o + v}
    .\end{equation}
    With this choice, we increase the probability of fooling the network on this point.
\end{rem}

This first method is local and does not take into account the curvature of the data. Hence, \revision{we propose a new method to improve the performances, especially in regions of $\X$ where this curvature is high.}

\section{A two-step attack}%
\label{sec:dogleg_attack}

The idea here is to take a first local step $v$ as in Section~\ref{sec:local_method}, and then take a second step to refine the linear approximation. The two-step nature of this attack allows the second step to take into account the curvature of the area around the input point. This method is able to exploit the geometry of the problem to construct a better approximation of the original problem \ref{eq:AdvAttack}. This attack will be named the Two Step Spectral Attack (TSSA) .

Let $v$ be the approximated solution of \ref{eq:AdvAttack} using the method of Section~\ref{sec:local_method} but with a budget of $\mu^2 < \varepsilon^2$.

The new problem to solve is the following:
\begin{equation}
\label{eq:step_2}
				\max_w \norm{w}_\X^2  \text{ subject to }
				\cas{\norm{w+v}_2^2 \leq \varepsilon^2 \\ 
								\norm{v}_2^2 = \mu^2 < \varepsilon^2 \\ 
								v \text{ eigenvector of }G_x } \tag{S2P}
\end{equation}

\rerevision{In practice,} one could solve this problem in the same way as in Section~\ref{sec:local_method} by taking $w$ the eigenvector of  $G_{x+v}$ with the largest eigenvalue once again (see \autoref{fig:2step_attack}). By doing that, we linearly approach the geodesic between $x_o$ and $x_a$ by \revision{two line segments in the Euclidean sense}. Taking more steps would approximate the geodesic better. Using only two steps as proposed is a simple procedure from a computational point of view and still achieves significantly better results than taking just one step as we will see in Section~\ref{sec:example}.

\begin{figure}[ht]
    \centering
				\def\svgwidth{0.5\columnwidth}
			%% Creator: Inkscape inkscape 0.92.5, www.inkscape.org
%% Accompanies image file 'two-step.pdf' (pdf, eps, ps)
\begingroup%
  \makeatletter%
  \providecommand\color[2][]{%
    \errmessage{(Inkscape) Color is used for the text in Inkscape, but the package 'color.sty' is not loaded}%
    \renewcommand\color[2][]{}%
  }%
  \providecommand\transparent[1]{%
    \errmessage{(Inkscape) Transparency is used (non-zero) for the text in Inkscape, but the package 'transparent.sty' is not loaded}%
    \renewcommand\transparent[1]{}%
  }%
  \providecommand\rotatebox[2]{#2}%
  \newcommand*\fsize{\dimexpr\f@size pt\relax}%
  \newcommand*\lineheight[1]{\fontsize{\fsize}{#1\fsize}\selectfont}%
  \ifx\svgwidth\undefined%
    \setlength{\unitlength}{484.61043711bp}%
    \ifx\svgscale\undefined%
      \relax%
    \else%
      \setlength{\unitlength}{\unitlength * \real{\svgscale}}%
    \fi%
  \else%
    \setlength{\unitlength}{\svgwidth}%
  \fi%
  \global\let\svgwidth\undefined%
  \global\let\svgscale\undefined%
  \makeatother%
  \begin{picture}(1,0.7353281)%
    \lineheight{1}%
    \setlength\tabcolsep{0pt}%
    \put(0,0){\includegraphics[width=\unitlength,page=1]{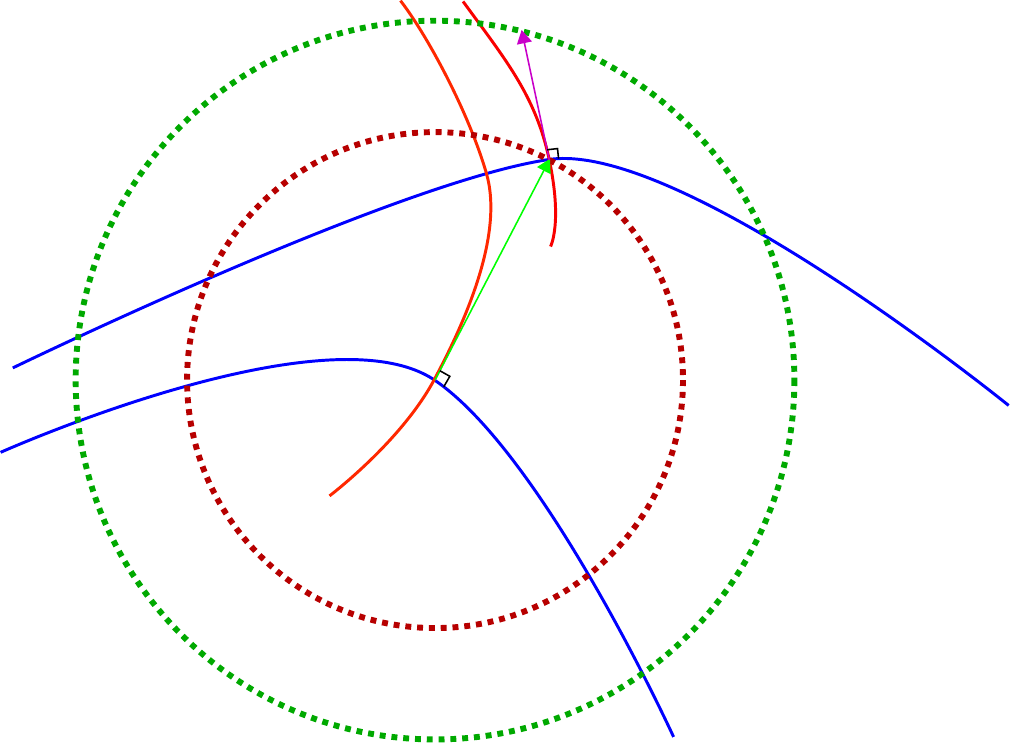}}%
    \put(0.50865986,0.46195035){\color[rgb]{0,1,0}\transparent{0.46666667}\makebox(0,0)[lt]{\lineheight{1.25}\smash{\begin{tabular}[t]{l}$v$\end{tabular}}}}%
    \put(0.544857,0.64112927){\color[rgb]{0.78431373,0,0.78431373}\makebox(0,0)[lt]{\lineheight{1.25}\smash{\begin{tabular}[t]{l}$w$\end{tabular}}}}%
    \put(0.23522912,0.22621148){\color[rgb]{0.71764706,0,0}\makebox(0,0)[lt]{\lineheight{1.25}\smash{\begin{tabular}[t]{l}$\mu$\end{tabular}}}}%
    \put(0.18106188,0.12229883){\color[rgb]{0,0.6627451,0}\makebox(0,0)[lt]{\lineheight{1.25}\smash{\begin{tabular}[t]{l}$\varepsilon$\end{tabular}}}}%
    \put(0.50740312,0.71965024){\color[rgb]{0,0,0}\makebox(0,0)[lt]{\lineheight{1.25}\smash{\begin{tabular}[t]{l}$x_a$\end{tabular}}}}%
    \put(0.45000324,0.35498444){\color[rgb]{0,0,0}\makebox(0,0)[lt]{\lineheight{1.25}\smash{\begin{tabular}[t]{l}$x_o$\end{tabular}}}}%
  \end{picture}%
\endgroup%

    \caption{The two-step attack in action. The two circles represent the Euclidean budget. The blue curves represent the leaves of the kernel foliation and the red curves represent the transverse foliation.}
    \label{fig:2step_attack}
\end{figure}

To explicit the action of the curvature at $x$ of the input space on the trajectory of this multi-step attack, we will approximate $G_{x+v}$ by its value at $x$ using normal coordinates\footnote{see Definition~\ref{def:normal_coord} in~\ref{app:intro_geo}}.
\begin{rem}
In the sequel of the article, for any vector (or tensor) $x$, $\bar{x}$ will denote $x$ expressed in terms of normal coordinates. Besides, we are going to use the Einstein summation notation omitting the symbol $\sum$ whenever the index over which the sum should apply repeats.
\end{rem}

\begin{prop}
If $\overline{x}$ are the normal coordinates at $x_o$ and if $R$ is the Riemannian curvature tensor, then
\begin{equation}
\label{eq:approx_normal_metric}
w^T G_{x+v} w = \der{x^m}{{\overline{x}}^i} \der{x^n}{{\overline{x}}^j} \pa{\delta_{mn} + \frac{1}{3} \overline{R}_{m k l n}(x)\overline{v}^k\overline{v}^l} w^iw^j + o\pa{\norm{{v}}^2}
\end{equation} %TODO: vérifier le + o()

\end{prop}

\begin{proof}
\begin{align*}
				  w^T G_{x+v} w &=  \pa{G_{x+v}}_{ij} w^i w^j \\
																	 &=  \der{x^m}{{\overline{x}}^i} \der{x^n}{{\overline{x}}^j} \pa{\overline{G}_{x+v}}_{nm} w^iw^j \\
																	 &=   \der{x^m}{{\overline{x}}^i} \der{x^n}{{\overline{x}}^j} \pa{\delta_{mn} + \frac{1}{3} \overline{R}_{m k l n}(x)\overline{v}^k\overline{v}^l+ o\pa{\norm{{v}}^2}} w^iw^j
.\end{align*}
The last line is obtained by the second-order Taylor expansion of $g$ in normal coordinates centered at $x$. The reader can find details of the computations in~\cite{willmore2013introduction}, Section 3.5 Corollary 7.
\end{proof}

\begin{prop}
By denoting $\overline{w} = \pa{\der{x^m}{{\overline{x}}^i}w^i}_m = Pw$ and $R_v = \overline{R}_{m k l n}(x)\overline{v}^k \overline{v}^l$, \autoref{eq:approx_normal_metric} can be rewritten with matrix notation by the following:
\begin{equation}
\label{eq:approx_norm}
\norm{w}_\X^2 = \norm{\overline{w}}_2^2 + \frac{1}{3} {\overline{w}}^T R_v \overline{w} + o\pa{\norm{v}^2} %TODO: vérifier
.
\end{equation}
\end{prop}

\begin{rem}
				In what follows, $\norm{w}_\X^2$ will always be taken at $x+v$, and be approximated by the right hand of \autoref{eq:approx_norm}. Additionally, we compare $w$ and $v$ without taking into account parallel transport since the Christoffel symbols vanish around the origin in normal coordinates.
\end{rem}

The transition matrix $P$ is equal to $\pa{\der{x^i}{{\overline{x}}^j}}_{i,j}$ and its inverse to $\pa{\der{\overline{x}^i}{{x}^j}}_{i,j}$. We should have for instance: \[G_{x} = P^T \overline{G}_x P = P^T I_n P \text{ and } {P^{-1}}^T G_x P^{-1} = I_n\]

The matrix $P^{-1} = \bmat{\frac{v_1}{\sqrt{\lambda_1}} & \cdots & \frac{v_n}{\sqrt{\lambda_n}}}$ with $v_i$ the eigenvector of $G_x$ associated with the eigenvalue $\lambda_i$ satisfies this equation (the family is chosen to be orthonormal for the ambient Euclidean metric: $v_i^T v_j = \delta_{i,j}$).
Note that this gives us \[P = \bmat{\sqrt{\lambda_1} v_1^T \\ \vdots \\ \sqrt{\lambda_n} v_n^T }.\]

\begin{rem}
				The pullback metric $g_x$ is always degenerate in this problem. Indeed, the dimension of its image is strictly bounded by the number of classes of the given task\footnote{See \cite{grementieriModelcentricDataManifold2021} for the proof.}.

				To take this into account, if $d= \dim\Ima G_x$, one can rewrite the metric in normal coordinates by:
				\[
								\overline{G}_x = \bmat{\textbf{I}_{d} &  \textbf{0}_{d,n-d} \\ \textbf{0}_{n-d,d} & \textbf{0}_{n-d,n-d}}
				.\]
								The transition matrix $P^{-1}$ is equal to  $\bmat{\frac{v_1}{\sqrt{\lambda_1}} & \cdots & \frac{v_d}{\sqrt{\lambda_d}} & v_{d+1} & \cdots & v_n}$.
\end{rem}

\begin{prop}
    If $w$ is a solution to \ref{eq:step_2}, then there exists a scalar $\lambda\geq0$ such that
    \begin{equation}
    \label{eq:step_2_sol}
    \pa{P^TBP - \lambda I_n}w = \lambda v
    \end{equation}
    with $B = I_n + \frac{1}{3}R_v$.
\end{prop}

\begin{proof}
Using again the KKT conditions but with the constraint $\norm{v+w}_2^2 = \norm{v+ P^{-1} \overline{w}}_2^2 \leq  \varepsilon^2$, it implies that there exists a scalar $\lambda\geq 0$ such that:
\begin{align*}
				& \nabla_{\overline{w}} \pa{\norm{\overline{w}}_2^2 + \frac{1}{3}\overline{w}^T R_v \overline{w}}\\
				& - \lambda \nabla_{\overline{w}} \pa{\norm{v + P^{-1}\overline{w}}_2^2 - \varepsilon^2} =0 \\
				\implies & \overline{w} + \frac{1}{3}R_v \overline{w} - \lambda\pa{{P^{-1}}^T\pa{P^{-1} \overline{w} + v}} =0
.\end{align*} 

Thus

\begin{align*}
				& P^T\pa{I_n + \frac{1}{3}R_v}Pw =  \lambda \pa{w+v} \\
				\iff & P^T\pa{B - \lambda {P^{-1}}^TP^{-1}}Pw = \lambda v \\
				\iff & \pa{P^TBP - \lambda I_n}w = \lambda v
\end{align*} 
where $B = I_n + \frac{1}{3}R_v$.
\end{proof}

\begin{cor}
				In normal coordinates, \autoref{eq:step_2_sol} rewrites as:
				\[
				PP^TB\overline{w} = \lambda \pa{\overline{w} + \overline{v}} \text{ ie }\pa{P P^TB-\lambda I_n} \overline{w} = \lambda \overline{v}
				.\] 
\end{cor}

\begin{rem}
Note that $P P^T = \diag\pa{\lambda_i}$.
\end{rem}

\begin{cor}
Whenever $\lambda$ is not an eigenvalue of $P^TBP$,  $\pa{P^TBP-\lambda I_n}$ is non-singular and 
\begin{equation}
w = \lambda  \pa{P^TBP-\lambda I_n}^{-1}  v.
\end{equation}
\end{cor}

It remains to find a $\lambda$ such that the constraint $\norm{w+v}_2^2 \leq \varepsilon^2$ is satisfied. There are two cases:

\begin{enumerate}
				\item either $\norm{w+v}_2^2 = \varepsilon^2$ and $\lambda > 0$,
				\item or $\norm{w+v}_2^2 < \varepsilon^2$ and $\lambda = 0$.
\end{enumerate}

\subsection{The case \texorpdfstring{$\lambda>0$}{λ>0}}
\label{sub:lambda>0}

We will use the constraint $\norm{w+v}_2^2 = \varepsilon^2$ to get $\lambda$. We suppose in what follows that $2\lambda$ is not an eigenvalue of $B$, ie  $(2\lambda-1)$ is not an eigenvalue of $\frac{1}{3}R_v$. The constraint may be written as:
\begin{align*}
				\norm{w+v}_2^2 &= \norm{\lambda\pa{P^TBP-\lambda I_n}^{-1} v + v}_2^2 \\
			    & = \norm{ \pa{\lambda \pa{P^TBP-\lambda I_n}^{-1} + I_n}v}_2^2 
.\end{align*}

\begin{lem}
				 \[
								\lambda\pa{P^TBP-\lambda I_n}^{-1} + I_n = \pa{P^TBP - \lambda I_n}^{-1}P^TBP
				.\] 
\end{lem}

\begin{proof}
				\begin{align*}
								&\pa{P^TBP - \lambda I_n} \pa{\lambda\pa{P^TBP-\lambda I_n}^{-1} + I_n }\\ =& \lambda \pa{P^TBP - \lambda I_n}\pa{P^TBP-\lambda I_n}^{-1} + \pa{P^TBP - \lambda I_n} \\
								=& \cancel{\lambda I_n} + P^TBP - \cancel{ \lambda I_n }
								= P^TBP
				.\end{align*}
\end{proof}

Thus,

\begin{equation}
\label{eq:eps_lambda}
				\varepsilon^2 = \norm{w+v}_2^2 = \norm{P^TBP\pa{P^TBP-\lambda I_n}^{-1}v}_2^2 
.\end{equation}

To find $\lambda$ that satisfies \autoref{eq:eps_lambda}, we will study the vanishing points of: 
\[
				\fun{\varphi}{[0,\infty[}{[0,\infty[}{\lambda}{ \norm{P^TBP\pa{P^TBP-\lambda I_n}^{-1}v}_2^2 - \varepsilon^2}
\] 

However, finding the vanishing points of such a function is not an easy task. Several methods may be used. A numerical method such as the Newton's method \cite{dennis1996numerical} could be applied.

Alternatively, observe that problem \ref{eq:step_2} can be simplified by using the triangular inequality to get back to an easier eigenvalue problem. Indeed, consider the following problem:

\begin{equation}
\label{eq:step_2_strong}
				\max_w \norm{w}_\X^2  \text{ subject to }
				\cas{\norm{w}_2 \leq \varepsilon - \mu\\ 
								\norm{v}_2^2 = \mu^2 < \varepsilon^2 \\ 
								v \text{ eigenvector of }G_x } \tag{S2.2P}
\end{equation}

This new problem is illustrated on \autoref{fig:2step_attack_tri}. The green circle is the true budget and the two other smaller circles represent the triangular inequality approximation. One can see that the second step $w$ does not reach the green circle but stop before due to the approximation.

\begin{figure}[ht]
    \centering
				\def\svgwidth{0.5\columnwidth}
			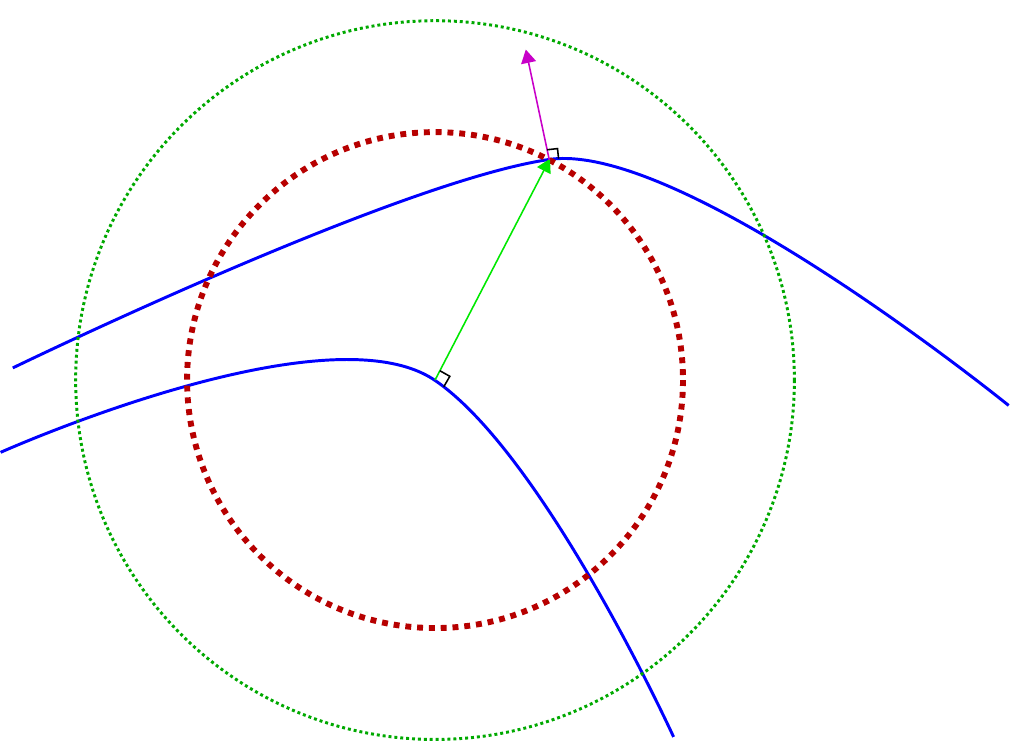

    \caption{The two-step attack in action with the triangular inequality simplification. The three circles represent the Euclidean budget. The blue curves represent the leaves of the kernel foliation.}
    \label{fig:2step_attack_tri}
\end{figure}

\begin{prop}
Any solution of \ref{eq:step_2_strong} problem will satisfy the constraint of \ref{eq:step_2} problem. 
\end{prop}

\begin{proof}
Indeed, by the triangular inequality, $\norm{w + v}_2 \leq \norm{w}_2 + \norm{v}_2 \leq \pa{\varepsilon - \mu} + \mu = \varepsilon$.
\end{proof}

With this new constraint, KKT's conditions boil down to

\[
P^T B P w = \lambda w
\] 
or in normal coordinates:
\[
PP^TB \overline{w} = \lambda \overline{w}
.\] 

\begin{prop}\label{prop:solution_TSSA}
A solution to \ref{eq:step_2_strong} is to choose $\overline{w}$ to be the eigenvector of $PP^TB$ with the highest eigenvalue $\lambda$ and with the appropriate Euclidean norm of $\varepsilon - \mu$.
\end{prop}

\rerevision{
An algorithm that finds a solution to \ref{eq:step_2_strong} can then be created using Proposition~\ref{prop:solution_TSSA}. It is presented in Algorithm~\ref{algo:TSSA}. It is then implemented\footnote{\minrevision{The implementation of TSSA is available at \url{https://github.com/eliot-tron/CurvNetAttack}.}} and tested in Section~\ref{sec:num_res}.

\begin{algorithm}
\caption{Two Step Spectral Attack (TSSA)}\label{algo:TSSA}
\DontPrintSemicolon
\KwData{$x_o$ initial point, $\varepsilon > \mu > 0$ euclidean budgets.}
\KwResult{$x_a$ attacked point}
$v \gets $ highest eigenvector of $G_{x_o}$\;
$v \gets \mu \times \frac{v}{\norm{v}_2}$\;
$y_{x_o} \gets \argmax_y p_\theta\pa{y\mid x_o}$\;
\If{$p_\theta\pa{y_{x_o} \mid x_o+v} > p_\theta\pa{y_{x_o} \mid x_o-v}$}{
    $v\gets -v$\;
}
$w \gets $ highest eigenvector of $G_{x_o + v}$\;
$w \gets (\varepsilon - \mu) \times \frac{w}{\norm{w}_2}$\;
\If{$v^T w < 0$}{
    $w\gets -w$\;
}
$x_a \gets x_o + v + w$\;
\KwRet{$x_a$}
\end{algorithm}
}

\minrevision{
\subsection*{Complexity analysis of TSSA}
Let $d$ be the dimension of the input space of the neural network, and $C$ the number of classes of the classification problem.
The runtime analysis of Algorithm~\ref{algo:TSSA} can be understood as follow:
\begin{itemize}
    \item[-] Line 1: 
        \begin{enumerate}
            \item Computing $G_{x_o}$: It requires the computation of the the network's Jacobian matrix with respect to the input, as well as inverting the output of the neural network. Computing such Jacobian matrix requires $O(C)$ calls to the Autograd procedure, PyTorch automatic differentiation package \cite{baydin2018AutomaticDiff}, where $C$ is the dimension of the output. The complexity of the Autograd algorithm depends on the neural network architecture, and scales with the number of weights, but also with the dimension of the input. After that, $2$ matrix multiplications are needed to get $G_{x_o} = J_N(x)^T {\diag\pa{N(x)}}^{-1} J_N(x)$.
            \item Computing the highest eigenvector of this $d\times d$ matrix: It can be solved with a power method\footnote{See Section~7.3, p.~365 of \cite{golub2013matrix}.} in $O\pa{d^2}$.
        \end{enumerate}
    \item[-] Line 2: $O(d)$.
    \item[-] Line 3: Evaluation time of the neural network, and then finding the maximum in $O(C)$.
    \item[-] Line 4: Two evaluations of the neural network.
    \item[-] Line 5: $O(d)$
    \item[-] Line 7: Same complexity as Line~1.
    \item[-] Lines 8-12: $O(d)$.
\end{itemize}

In practice, Line 1 and 7 are the most expensive ones. This algorithm requires, in total, $O(C)$ calls to Autograd with respect to the input of the network for each given input point $x_o$, which limits the tests on large datasets requiring large networks. Possible runtime improvements can be developed in future work such as the use of faster power iteration methods (e.g. Lanczos), or the use of the alias method when $C$ is high. See \cite{zhaoAdversarialAttackDetection2019} for some details on these aspects. Some other improvements can be developed on the implementation side in PyTorch with batched gradient computation. We leave this for future work and focus here on the effect of curvature on adversarial attacks.
}

\subsection{The case \texorpdfstring{$\lambda = 0$}{λ=O}}

In that case, $w$ is in the interior of the boundary of the problem. The problem reduces to
 \begin{equation*}
				B \overline{w} = 0
				 \iff R_v \overline{w} = -3 \overline{w}
.\end{equation*}

This means that $w$ is the optimal second step if $\overline{w}\in\Ker B$, ie if $w$ is eigenvector of $R_v$ with eigenvalue $-3$. However, this case does not produce any interesting KKT admissible point. Indeed, 
\begin{align*}
				\norm{w}_\X^2 & = \norm{\overline{w}}_2^2 + \frac{1}{3}\overline{w}^TR_v \overline{w}
																				= \norm{\overline{w}}_2^2 - \overline{w}^T \overline{w} 
																				= 0
.\end{align*}

Therefore, the case $\lambda=0$ does not lead to useful adversarial attacks when the previous approximations are applied to the problem. The study of this singularity is left as future work.

\section{An enlightening play-test example}%
\label{sec:example}

\subsection{Setup}
In this section, we focus on a low dimensional example in order to visualize more easily the effect of curvature on the efficiency of the adversarial attack. A simple neural network $N_\theta$ with one hidden layer of $k$ neurons and sigmoids as activation functions is used. Its architecture is depicted on \autoref{fig:XorNetSchema}.

    \begin{figure}[ht]
				\centering
				\def\layersep{2.5cm}

\begin{tikzpicture}[shorten >=1pt,->,draw=black!50, node distance=\layersep]
				\tikzstyle{every pin edge}=[<-,shorten <=1pt]
				\tikzstyle{neuron}=[circle,fill=black!25,minimum size=17pt,inner sep=0pt]
				\tikzstyle{input neuron}=[neuron, fill=green!50];
				\tikzstyle{output neuron}=[neuron, fill=red!50];
				\tikzstyle{hidden neuron}=[neuron, fill=blue!50];
				\tikzstyle{annot} = [text width=4em, text centered]

				% Draw the input layer nodes
				\foreach \name / \y in {1,...,2}
				% This is the same as writing \foreach \name / \y in {1/1,2/2,3/3,4/4}
				\node[input neuron, pin=left:$x_\y$] (I-\name) at (0,-\y) {};

				% Draw the hidden layer nodes
				\foreach \name / \y in {1,...,3}
				\path[yshift=0.5cm]
				node[hidden neuron] (H-\name) at (\layersep,-\y cm) {};

				% Draw the output layer node
				\foreach \name / \y in {1}
				\node[output neuron,pin={[pin edge={->}]right:$N_\theta(x)$}] (O-\name) at (2*\layersep, -0.5-\y) {};

				% Connect every node in the input layer with every node in the
				% hidden layer.
				\foreach \source in {1,...,2}
				\foreach \dest in {1,...,3}
				\path (I-\source) edge (H-\dest);

				% Connect every node in the hidden layer with the output layer
				\foreach \source in {1,...,3}
				\foreach \dest in {1}
				\path (H-\source) edge (O-\dest);

				% Annotate the layers
				\node[annot,above of=H-1, node distance=1cm] (hl) {Hidden layer};
				\node[annot,left of=hl] {Input layer};
				\node[annot,right of=hl] {Output layer};
				\draw [decorate,decoration={brace,amplitude=5pt,mirror,raise=10pt},-, thick] (0.3,-2.5) -- (-0.3cm + \layersep, -2.5) node [black,midway,yshift=-0.8cm] {$W_1$};
				\draw [decorate,decoration={brace,amplitude=5pt,mirror,raise=10pt},-, thick] (0.3cm + \layersep,-2.5) -- (-0.3cm + 2*\layersep, -2.5) node [black,midway,yshift=-0.8cm] {$W_2$};
\end{tikzpicture}
				\caption{XorNet $N_\theta$ with 3 hidden neurons.}
				\label{fig:XorNetSchema}
    \end{figure}

To be more precise, if $\sigma:a\in\R^d \mapsto \pa{\frac{1}{1+e^{-a_i}}}_{i=1}^d \in \R^d$ is the sigmoid function and if $x\in\R^2$, $W_1\in\M_{2,k}(\R)$, $W_2\in\M_{k,1}(\R)$, we have %TODO: softmax or not ?
\begin{equation}
				N_\theta(x) = \sigma\pa{W_2\sigma\pa{W_1x+b_1}+b_2}
.
\end{equation}

This neural network is then trained to approximate the very simple function \texttt{Xor}$:\set{0,1}^2\to\set{0,1}$ defined in \autoref{tab:Xor}.

    \begin{table}[ht]
				\centering
				\caption{\texttt{Xor} function.}
				\label{tab:Xor}
				\begin{tabular}{c|cc}
								\texttt{Xor} & 0 & 1 \\
								\hline
								0 & 0 & 1 \\
								1 & 1 & 0
				\end{tabular}
    \end{table}
The output $N_\theta(x)$ is seen as the parameter $p$ of a Bernoulli law and associates to the network the following probability distribution:  $p(y\mid x,\theta)$ where  $\theta$ is the vector containing the weights and biases $\pa{W_i,~b_i}$,  $x$ is the input and  $y$ is the true label.

\begin{prop}
The random variable $Y\mid X,\theta$ follows Bernoulli's law of parameter $p = N_\theta(x)$.
\end{prop}

\subsection{Computing the output FIM}

The output of the network is the manifold of Bernoulli probability densities parameterized by the open segment $]0,1[$. 

\begin{prop}
Let $p\in]0,1[$ and $G_p$ the Fisher Information Metric at the point $p$.
\begin{equation}
G_p = \frac{1}{p} \frac{1}{1-p}
\end{equation}
\end{prop}

\begin{proof}
\begin{align*}
				G_p &= - \E_{y \mid x,\theta}\bra{\partial_p^2 \pa{\ln P(y\mid x,\theta)}} \\
						& = - \E\bra{\partial_p^2 \pa{y \ln p + \pa{1-y} \ln \pa{1-p}}} \\
						&= - \E\bra{\partial_p \pa{\frac{y}{p} - \frac{\pa{1-y}}{1-p}}} \\
						& = - \E\bra{-\frac{y}{p^2} - \frac{\pa{1-y}}{\pa{1-p}^2}} \\
						&= \frac{1}{p} + \frac{1}{1-p} 
						= \frac{1}{p} \frac{1}{1-p}.\quad
\end{align*}
\end{proof}

\subsection{Computing the pullback metric}

Let $x\in\M$ be a point associated with $p$ by the network.

\begin{lem}
The Fisher Information Metric $G_x$ on $\X$ is the pullback metric of $G_p$ by the neural network $N_\theta$.
\end{lem}

\begin{cor}
If $J = \bra{\der{p}{x_j}}_{j=1,2} = \bmat{\der{N_\theta(x)}{x_1} & \der{N_\theta(x)}{x_2}}$, then
\begin{equation}
				G_x = J^T G_p J
.
\end{equation}
\end{cor}

\begin{rem}
    The Jacobian matrix of a neural network $J$ is not difficult to compute thanks to automatic differentiation available in most neural network training software's packages.
    Besides, this approach for computing $J$ numerically allows the sequel of the article to stay quite general regarding the architecture of the neural network.
\end{rem}

Knowing how to compute the metric at any point unlocks the computation of the local method seen in Section~\ref{sec:local_method} and the two-step method seen in Section~\ref{sec:dogleg_attack} when the FIM is recomputed at the intermediary point.
Additionally, studying the curvature of the input space is insightful to understand the behavior of the correction step in the two-step attack, and more generally to understand why adversarial attacks are, in some cases, so efficient.

Nonetheless, the pullback metric being almost always only semi-definite for machine learning tasks, it makes sense to consider its kernel. The curvature will have a decomposition term on the kernel of $g$ and a decomposition term on its orthogonal\footnote{Here, the Euclidean orthogonal is considered.}. To craft the attack, the orthogonal term will be the only curvature component of interest as stated at the end of Section~\ref{ssec:setup}. Thus in the following subsection, we compute the metric kernel for this neural network.

\subsection{The metric kernel foliation}

\begin{definition}
The kernel of a metric $G$ at the point $x$ is defined by
 \[
				\ker_x G = \set{X\in T_p\M \mid X^T G_x Y = 0,~ \forall Y \in T_p \M}
.\]
\end{definition}

This kernel defines an integrable distribution when Frœbenius' condition\footnote{For more details, see Chapter~1 of \cite{1988riemannian}} is satisfied, and with this distribution emerges a Riemannian foliation on the input manifold, foliation defined by the action of the neural network.

\begin{lem}
	$\ker_x G = \ker J$.
\end{lem}

\begin{proof}
				We omit $\frac{1}{p}\frac{1}{1-p}$ during the proof because it is always non-zero.
				\begin{itemize}
								\item Let us prove first that $\ker_x G \subset \ker J$. If $X\in \ker_x G$, then $X^T J^T J X = 0$. Hence $\pa{JX}^T\pa{JX} = 0$, or written otherwise:  $\norm{JX} = 0$. Hence  $JX = 0$ and  $X\in\ker J$.
								\item Then we can prove that $\ker J \subset \ker_x G$. This inclusion is simply due to the fact that if  $X\in\ker J$, we have for all $Y \in T_p\M$ that  $X^TJ^TJY = \underbrace{\cancel{\pa{JX}^T}}_{=0} JX = 0$.\quad
				\end{itemize}
\end{proof}

\begin{prop}
\label{prop:distribution}
If at least one of the two components of $J$ is non-zero, the distribution at $x$ is one dimensional and given by $P_x = \Span\pa{J_2(x)\partial_1 - J_1(x)\partial_2}$.
\end{prop}

\begin{rem}
				If $J = 0$, the leaf at $x$ is singular and is of dimension 2.
\end{rem}

With the \texttt{Xor} function, the dimension of the leaves is $1$, thus the condition of Frœbenius is trivially satisfied and $P$ is integrable. 
\begin{prop}
If $\gamma : t\in I \mapsto \gamma(t)\in\M$ is an integral curve for the distribution $P$, it satisfies the following ODE:
 \begin{align*}
	\gamma'(t)  & = J_2\pa{\gamma(t)} \partial_1 - J_1 \pa{\gamma(t)}\partial_2 \\
	     	    & = J_{\gamma(t)} \bmat{0 & -1 \\ 1 & 0}
.\end{align*} 

\end{prop}

It can be solved numerically quite easily with a finite difference method\footnote{See for example \cite{strikwerda2004finite}.}. \autoref{fig:fol} provides illustrations of the computed neural network kernel foliation with such a method for the \texttt{Xor} function (\autoref{fig:fol_xor}), and for the \texttt{Or} function (\autoref{fig:fol_or}).

\begin{figure}[ht]
     \centering
     \begin{subfigure}{0.49\textwidth}
         \centering
         \includegraphics[width=\textwidth]{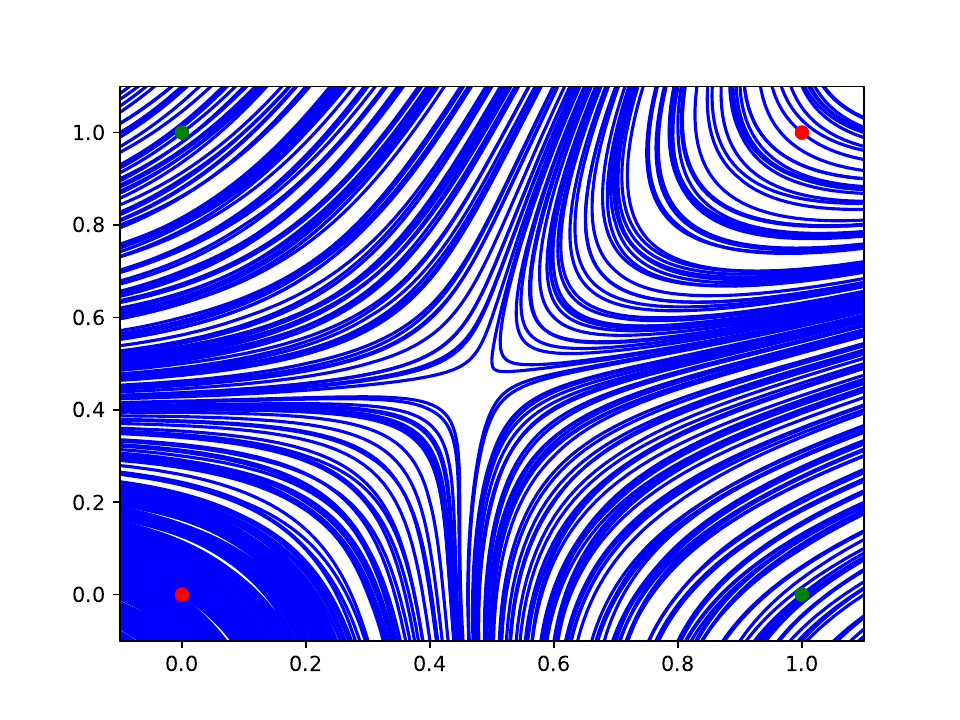}
         \caption{Task: \texttt{Xor}}
         \label{fig:fol_xor}
     \end{subfigure}
     \hfill
     \begin{subfigure}{0.49\textwidth}
         \centering
         \includegraphics[width=\textwidth]{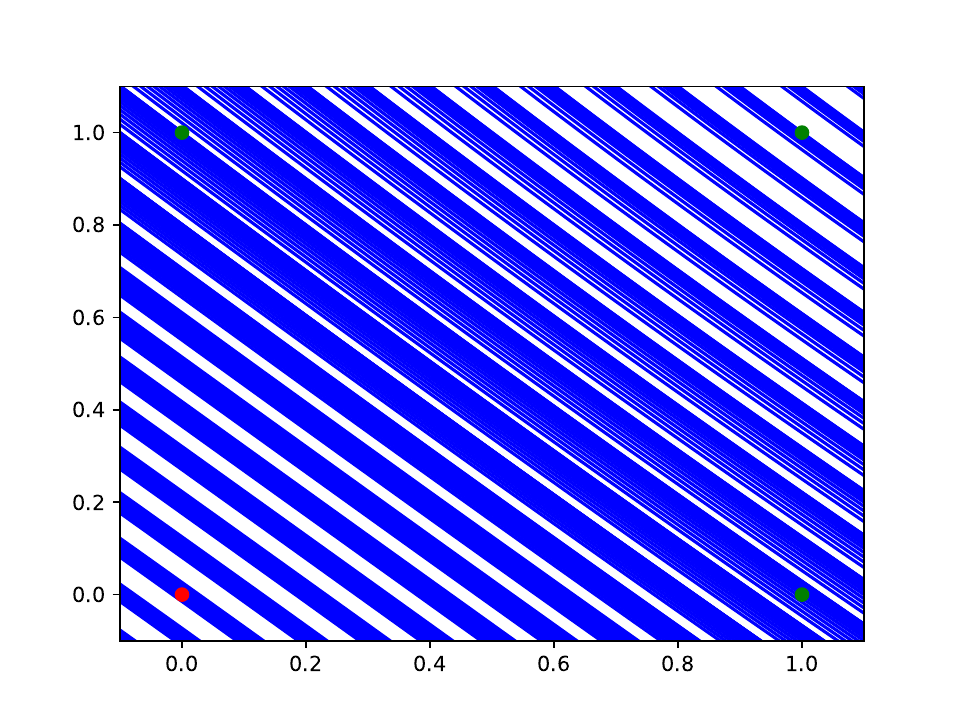}
         \caption{Task: \texttt{or}}
         \label{fig:fol_or}
     \end{subfigure}
        \caption{Kernel foliation: the leaves are represented by the blue lines, the red dots are the $0$ result and the green dots are the $1$ results.}
        \label{fig:fol}
\end{figure}

\begin{rem}
For readability, the transverse leaves are not represented but can be easily deduced from figure \ref{fig:fol} as Euclidean orthogonal curves to the blue ones.
\end{rem}

These results are quite interesting. First of all, we notice that a linearly separable problem such as the \texttt{or} function seems to have hyperplanes as leaves. In fact, it is easy to show that if the neural network is replaced by a linear form $x\mapsto \ang{n,x} + b$, then the leaves are hyperplanes defined by the normal vector $n$.

Second of all, one can see that there is a singular point close to the middle point $(0.5, 0.5)$ for the \texttt{Xor} task. Clearly, it is also easy to see that around that central point, the curvature of the leaves is the highest. Since the two-step adversarial attack makes use of the curvature, we can conjecture that it is in the central region that the second step will have the most impact. This phenomenon will be confirmed in Section~\ref{sec:num_res}.

\subsection{Visualizing the curvature}
\label{ssec:computing_the_curvature}

In this section, we are computing the extrinsic curvature of a transverse leaf to see where the two-step attack will differ the most from the one-step attack. But first, we compute the Levi-Civita connection. In the following of the article, we will use $i,j,k,l\in \nset{1}{\dim \X}$ as indices and we will use the Einstein summation notation.

\begin{definition}[Levi-Civita connection]
In coordinates $e_i$, if $\Gamma_{ij}^k$ are Chrystoffel's symbols, then the Levi-Civita connection is defined by:
\[\LC_{e_j} e_i = \Gamma_{ij}^ke_k\quad.\]
\end{definition}

\begin{definition}[Riemannian curvature]
The \emph{Riemannian Curvature Tensor} $R$ is defined such that for all three tensor fields $X,Y,Z$,
 \[
				R(X,Y)Z = \nabla_X \nabla_Y Z - \nabla_Y \nabla_X Z - \nabla_{\bra{X,Y}} Z
.\]

In local coordinates, this gives:
\begin{align*}
				R_{ijk}^l &= dx^l\pa{R\pa{\partial_i, \partial_j}\partial_k}\\
									&= dx^l \pa{\nabla_{\partial_i}\nabla_{\partial_j} \partial_k - \nabla_{\partial_j}\nabla_{\partial_i}\partial_k - \nabla_{\bra{\partial_i,\partial_j}}\partial_k} \\
									&= \der{\Gamma_{kj}^l}{x^i} - \der{\Gamma_{ki}^l}{x^j} + \Gamma_{kj}^{\alpha} \Gamma_{\alpha i}^l - \Gamma_{ki}^{\alpha}\Gamma_{\alpha j}^l
.\end{align*} 
\end{definition}

\begin{rem}
				In normal coordinates, at the origin point, one has:
				\[
								R_{ijkl} = \frac{1}{2} \pa{\partial_i \partial_l g_{jk} + \partial_j\partial_k g_{il} - \partial_i\partial_k g_{jl} - \partial_j \partial_l g_{ik}}
				.\] 
\end{rem}

Finding an explicit form for $R$ is unreasonable, but since we are in the simpler case where the dimension of the transverse leaves is 1, a quick approximation can do the trick. In fact, leaves are embedded submanifolds of the Euclidean space and the extrinsic curvature, i.e. the second fundamental form, can be computed as the rate of rotation of the normal vector. 

To approximate the second step $w$ of the two-step attack, we look at the rotation speed of the unit normal to the kernel leaf $\vec{n}$ when moved by an infinitesimal step $dx$. This infinitesimal step is taken in the direction of the transverse leaf, and is Euclidean. Since the rotation rate is approximated by finite difference, one can expect the angle variation to be very small. To ensure numerical stability, the usual procedure based on its cosine computation using inner product is replaced by one using the cross product. The next lemma gives the expression of the sine of the angle variation between two close positions on curve transverse to the $\ker g$ foliation.

\begin{lem}
If $\vec{n}_y$ is the normal to the kernel leaf at $y\in\X$, and if $\cdot \times \cdot $ is the cross product, then the infinitesimal variation of angle is
\[\abs{d\theta} = \arcsin\pa{\norm{\vec{n}_x \times \vec{n}_{x+dx}}_2}.\]
\end{lem}
Please note that in the small angles approximation, the sine can be replaced by the angle itself, thus recovering the usual infinitesimal rotation representation as a cross product (this is in fact a Lie algebra representation in the usual sense). 

To approximate the effect of the curvature during the Euclidean step $v$, one has to compute the rotation matrix $R$ of angle $d\theta$:
\[R = \bmat{\cos\pa{d\theta} & - \sin\pa{d\theta} \\ \sin\pa{d\theta} & \cos\pa{d\theta}}\]

\begin{prop}
If $v$ is the first step, the approximated second step is given by $w = R v$ and then re-normalized to get $\norm{w}^2 + \norm{v}^2 =  \varepsilon^2$.
\end{prop}

The signed curvature ${d\theta}$ for the \texttt{Xor} problem can then be seen on \autoref{fig:curvature_xor}. One can see that the curvature is the highest in the middle, around the point $(0.5,0.5)$, and also on the diagonals.

\begin{figure}[ht]
    \centering
    \includegraphics[width=0.8\textwidth]{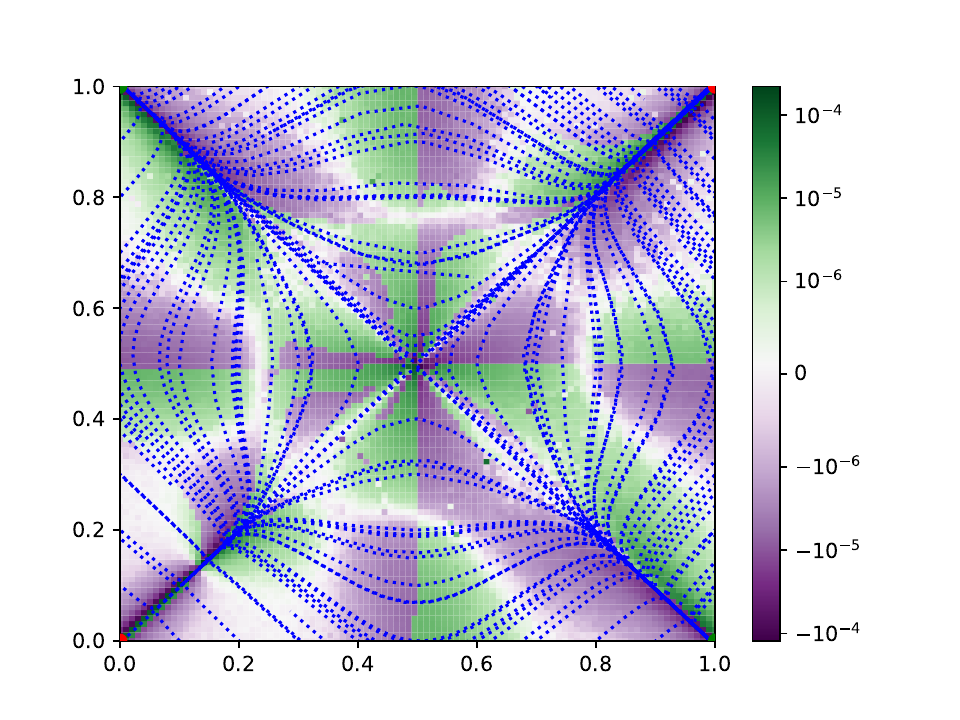}
		\caption{Extrinsic (signed) curvature of the transverse leaves (${d\theta}$) for the task \texttt{Xor}, computed with a $dx$ of $10^{-6}$, with the transverse leaves in blue.}
    \label{fig:curvature_xor}
\end{figure}

\section{Numerical results}
\label{sec:num_res}
\revision{In the following subsections, we first provide an illustrative example, and next we report experimental results on two public datasets: MNIST \cite{lecun1998mnist} and CIFAR10 \cite{cifar-10}.
}
All codes used to produce the following results can be found in~\cite{Tron_CurvNetAttack}. 

\minrevision{
Note that the first step in the Two Step Spectral Attack is set to have a Euclidean budget $\mu$ of $60\%$ of the total budget $\varepsilon$ in all following experiments.
}

\subsection{Xor dataset}

To compare the different attacks, we train a neural network with 8 hidden neurons on random points taken in the square $\bra{0,1}^2$ until convergence. We then compute the two different attacks which are the One Step Spectral Attack presented in Section~\ref{sec:local_method} and the Two Step Spectral Attack presented \rerevision{in Section~\ref{sec:dogleg_attack} with Algorithm~\ref{algo:TSSA}} where we compute the FIM at $x_o$ and $x_o+v$. These attacks are computed on 5000 random points selected in a square $I$ of variable length. The fooling rate is then computed as the quotient of the number of fooled prediction by the total number of points. The budget is selected between $0$ and $0.5$.  We plot the results on \autoref{fig:XOR_fooling_rates}. 

\begin{figure}[ht]
				\centering
				\begin{subfigure}{0.49\textwidth}
								\centering
								\includegraphics[width=\textwidth]{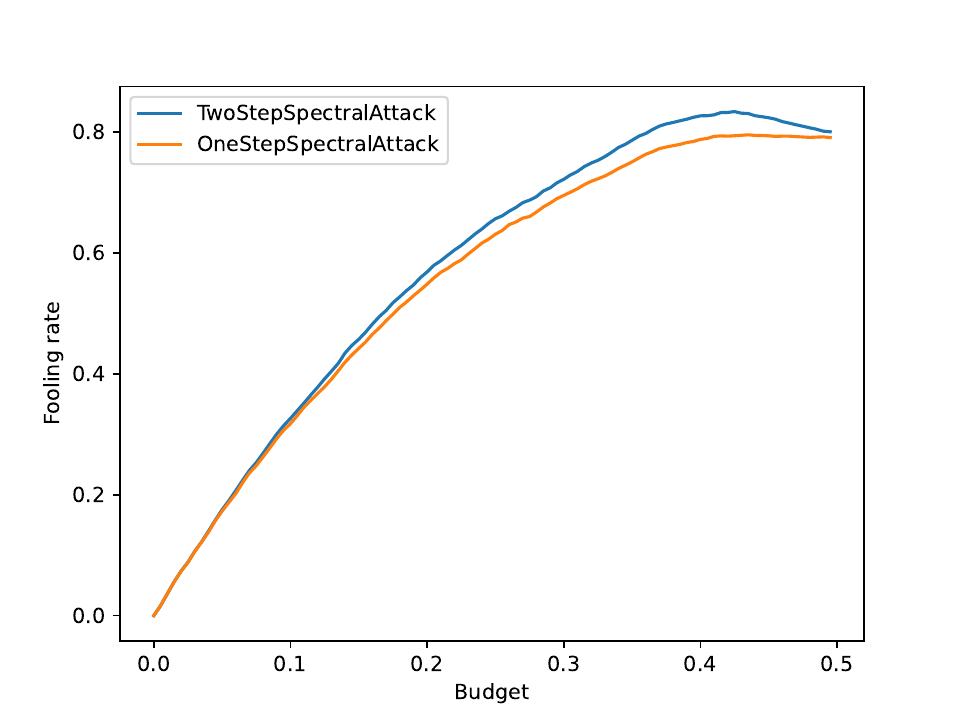}
								\caption{$I=\bra{0,1}^2$}
								\label{sfig:XOR_fr_full}
				\end{subfigure}
				\begin{subfigure}{0.49\textwidth}
								\centering
								\includegraphics[width=\textwidth]{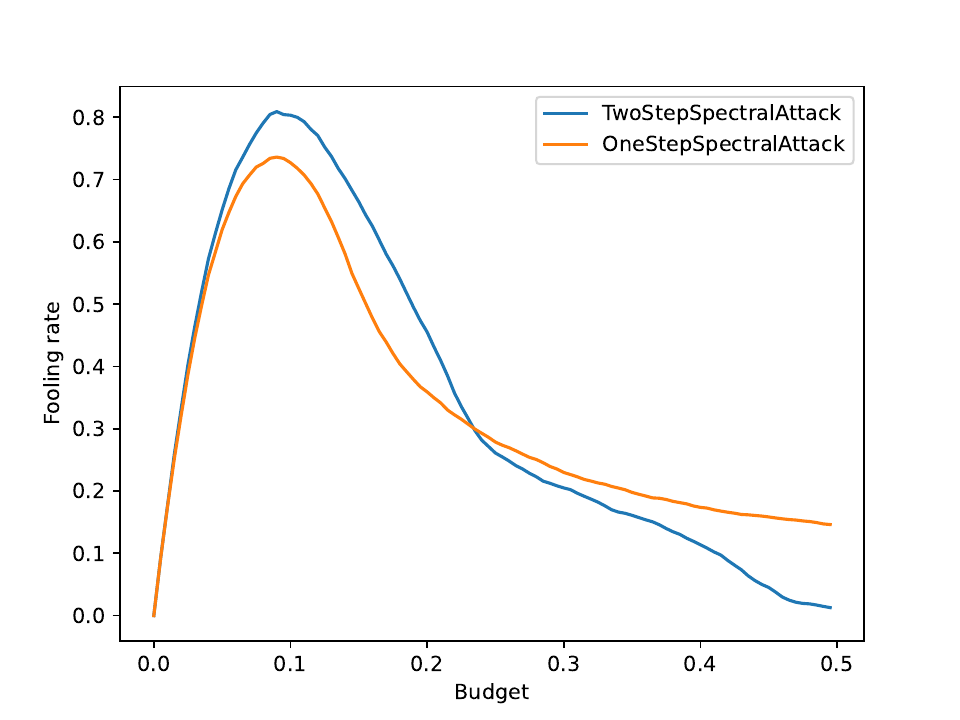}
								\caption{$I=\bra{0.4,0.6}^2$}
								\label{sfig:XOR_fr_center}
				\end{subfigure}
				\caption{Fooling rates with respect to the Euclidean budget with random points taken in $I$ for the task \texttt{Xor}.}
				\label{fig:XOR_fooling_rates}
\end{figure}

On each figure, the two-step attack beats the one step attack (for reasonable budgets). The two-step attack is especially strong on the area where the curvature is the strongest: almost one point better for the TSSA at the peak in high curvature zones (see Figure~\ref{sfig:XOR_fr_center}) compared to only half a point for the full space (see Figure~\ref{sfig:XOR_fr_full}). Indeed, as seen on Figure~\ref{fig:curvature_xor}, the curvature is really strong at the center of the square $[0,1]^2$. The eigenvector of the FIM associated with the greatest eigenvalue is always orthogonal to the leaf kernel. Therefore, striking close to the middle region without taking into account the curvature usually results in not changing the output label: $(0,0) \leftrightarrow (1,1)$ or $(1,0) \leftrightarrow (0,1)$. This is why the TSSA gets better results by taking the curvature into account.
\begin{rem}
Note that on Figure~\ref{sfig:XOR_fr_center}, the fooling rates collapse shortly after $\text{budget}=0.1$ because such high budgets makes every points leave the initial square $I$. These high budgets are too big for this area which is very close to the decision boundaries and the second step is not enough to compensate these instabilities, thus leading to such strange results at first glance.
\end{rem}

The illustration of the two different attacks can be found on \autoref{fig:Xor_plot_attacks}.

\begin{figure}
	\centering
	\includegraphics[width=0.8\textwidth]{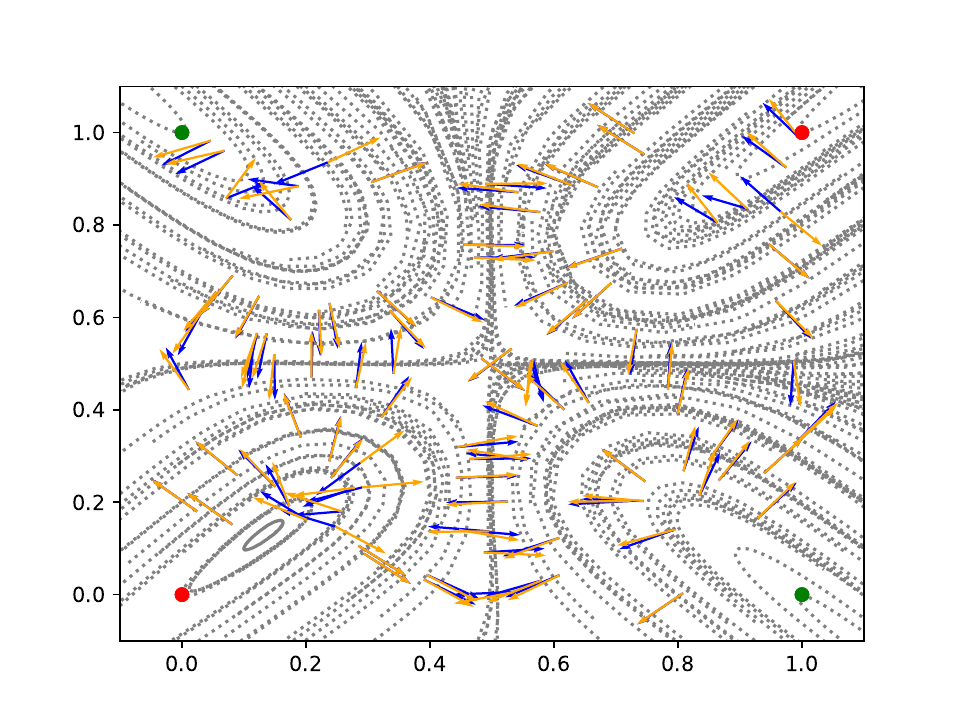}
\caption{TSSA (in \textcolor{blue}{blue}) compared to OSSA (in \textcolor{orange}{orange}) with $\varepsilon = 0.1$. The kernel foliation is depicted with the grey lines.}
	\label{fig:Xor_plot_attacks}
\end{figure}

These results confirm our intuition that using the information of local curvature to craft an adversarial attack is meaningful. It also highlights the role of the leaves of the kernel foliation in the sensitivity of neural networks to attacks.

\subsection{MNIST dataset}\label{ssec:MNIST}

Next, we train a medium convolutional neural network with 4 layers (2 convolutional and 2 linear), \texttt{ReLU} as activation function and a \texttt{Softmax} function at the output. The fooling rates for the One Step Spectral Attack and the Two Step Spectral Attack are represented on Figure~\ref{fig:fr_compared_MNIST} for different Euclidean budget ranging from 0 to 10. The TSSA performs better than the OSSA, proving that curvature of the network is of interest in the construction of attacks or defenses to these attacks. The difference between the fooling rate of the TSSA and the fooling rate of the OSSA is plotted on Figure~\ref{fig:fr_diff}. It represents the advantage given by the curvature. As one can see, there are 3 different regimes:
\begin{enumerate}
    \item The first one in blue is the region where the budget is small enough for the approximation $\exp_{x_o}(v) \approx x_o + v$ to hold. Taking into account the curvature does not modify enough the attack to fool the network on a lot of additional images compared to the straight line attack (OSSA).
    \item The second regime in green is the region where the budget is high enough for the correction due to the curvature to be essential, and yet small enough for this same correction to hold. That is why the TSSA is much better than the OSSA for these budgets.
    \item The third and last regime in red corresponds to budgets to big so that the approximations no longer make sense. At these horizons, the manifold is too non-linear.
\end{enumerate}

\revision{To evaluate our technique against State Of The Art (SOTA) attacks, we consider the AutoAttack (AA) proposed in~\cite{croce_reliable_2020}. The authors propose a method to create adversarial attacks combining and improving 4 SOTA algorithms, making the AutoAttack algorithm an appropriate candidate to compare performances to. Indeed, it has been designed to benchmark adversarial robustness, enabling the community to have a unified way to compare the performances of their adversarial defenses with others on SOTA adversarial attacks. To implement the method, we used the code available at \url{https://github.com/fra31/auto-attack.git}.
}
\minrevision{The settings used for the AutoAttack are the $L_2$ norm, a varying budget of $\varepsilon$, and the \emph{standard} version of the attack.}

\revision{
				The difference between the fooling rate of the TSSA and the fooling rate of the AutoAttack is plotted with respect to the Euclidean budget on Figure~\ref{fig:fr_diff_MNIST_TSSA-AA} for both \texttt{ReLU} and \texttt{Sigmoid} activation function cases. From the plots, one can see that the AutoAttack achieves better results than our Two Step attack for mid-range budgets. However, by design, the TSSA algorithm only performs two steps when the AutoAttack procedure carries out many moves (for example, 100 iterations for the APGD or FABAttack subroutines used within AutoAttack). To have a fair comparison with TSSA, it would require a 100 steps attack or equivalently solve the geodesic equation with 100 discretization steps. This is computationally expensive at each step due to the calculation of the FIM and its eigenvectors. Nevertheless, the TSSA achieves comparable performances for large budgets, and even better performances for small budgets. Note that the range of small budgets is the one considered in~\cite{croce_reliable_2020}. The results confirm that good adversarial attacks can be efficiently computed (meaning in very few steps) with the use of geometry, which was the primary objective of this study.
}

\revision{
Note that with equal Euclidean budget, the infinity norm of the attack is higher for the AutoAttack than for the TSSA or the OSSA most of the time as shown in Figure~\ref{fig:MNIST-inf_norm_compared_TSSA-OSSA-AA}. It means that, in the case of MNIST images, the absolute value of the variation of pixel intensity is higher in the case of AutoAttack for equal Euclidean budget. Therefore, attacks constructed with AutoAttack might be more noticeable to the human eye on the average.
}

Some examples of attacked images are also included for different budgets in order for the reader to realize that such small budgets are enough to fool the network while us - humans - still recognize very well the true digit. Other examples of attacks produced by the TSSA procedure are included in \autoref{fig:TSSA_bugdet_3} and in \autoref{fig:TSSA_bugdet_7} for different budgets.

\begin{figure}[ht]
    \centering
    \includegraphics[width=0.7\textwidth]{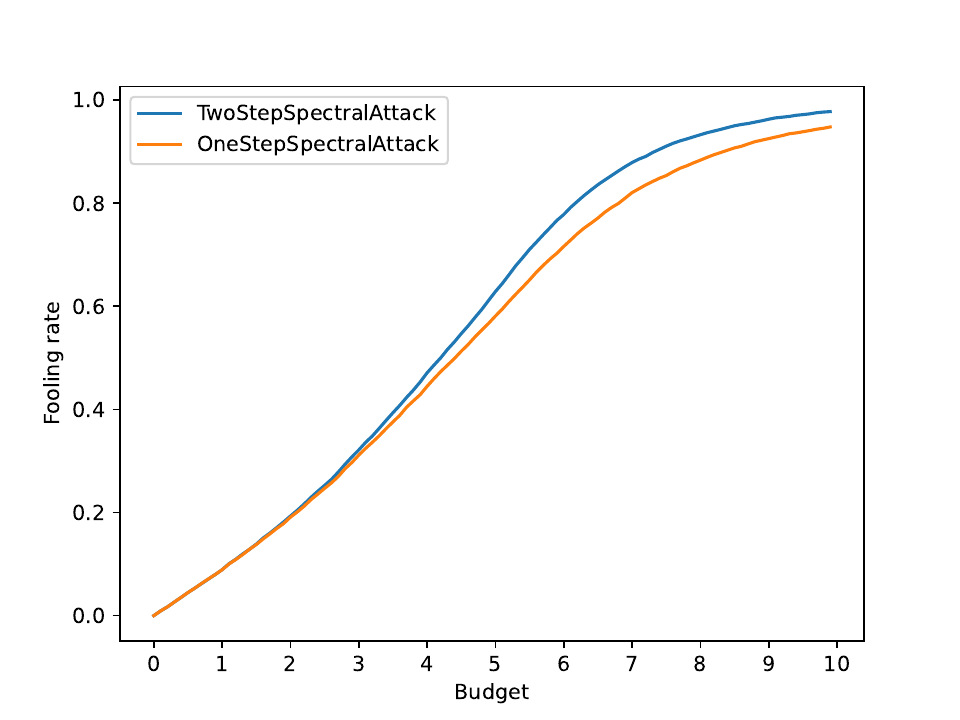}
    \caption{Fooling rates with respect to the Euclidean budget on all MNIST test-set.}
    \label{fig:fr_compared_MNIST}
\end{figure}

\begin{figure}[ht]
    \centering
    \includegraphics[width=\textwidth]{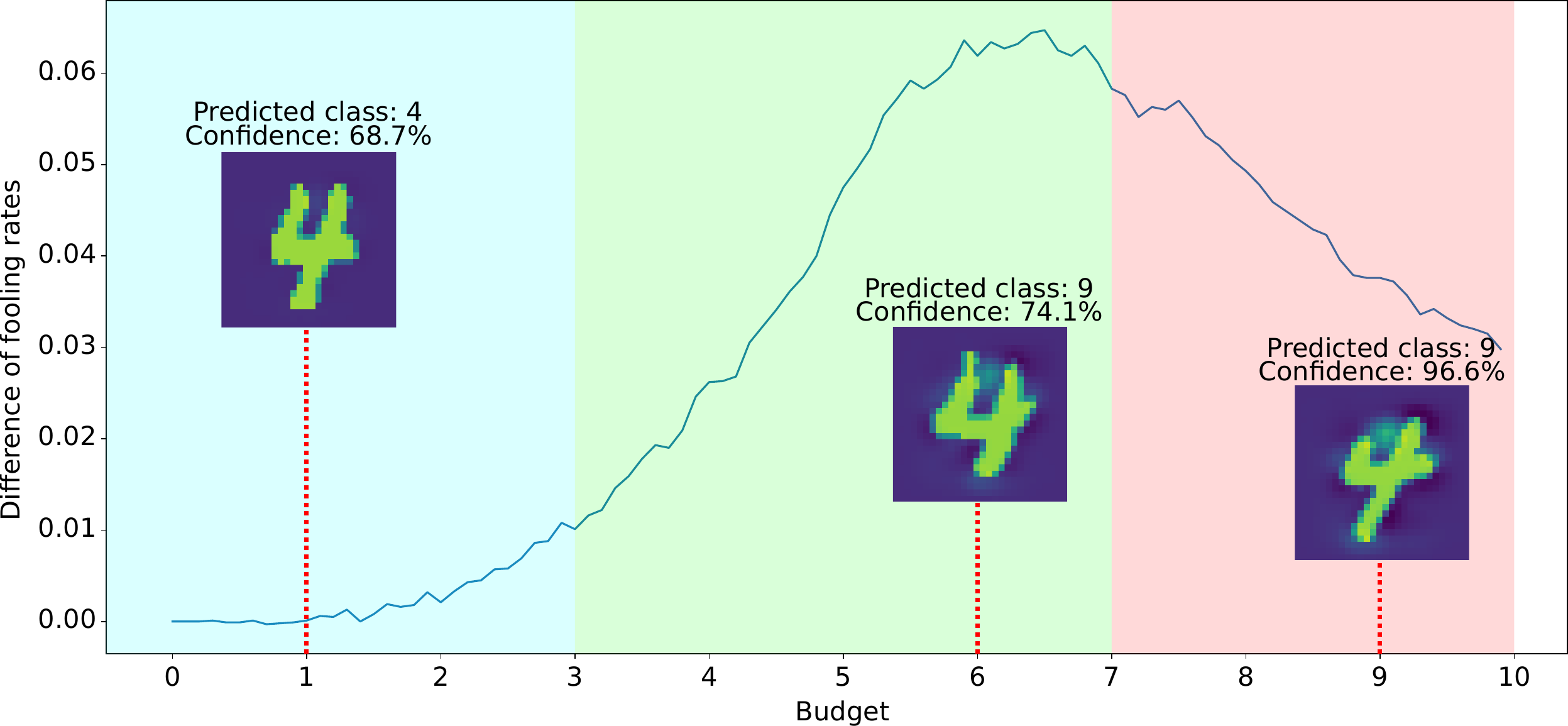}
    \caption{Difference between the fooling rate of the TSSA ($\text{fr}_{\text{TSSA}}$) and the one of the OSSA ($\text{fr}_{\text{OSSA}}$) with respect to the Euclidean budget ($\text{fr}_{\text{TSSA}} - \text{fr}_\text{OSSA}$), and some examples of the attacks it produces.}
    \label{fig:fr_diff}
\end{figure}

% Revision
\begin{figure}
	\centering
	\begin{subfigure}{\textwidth}
				\centering
				\includegraphics[width=\textwidth]{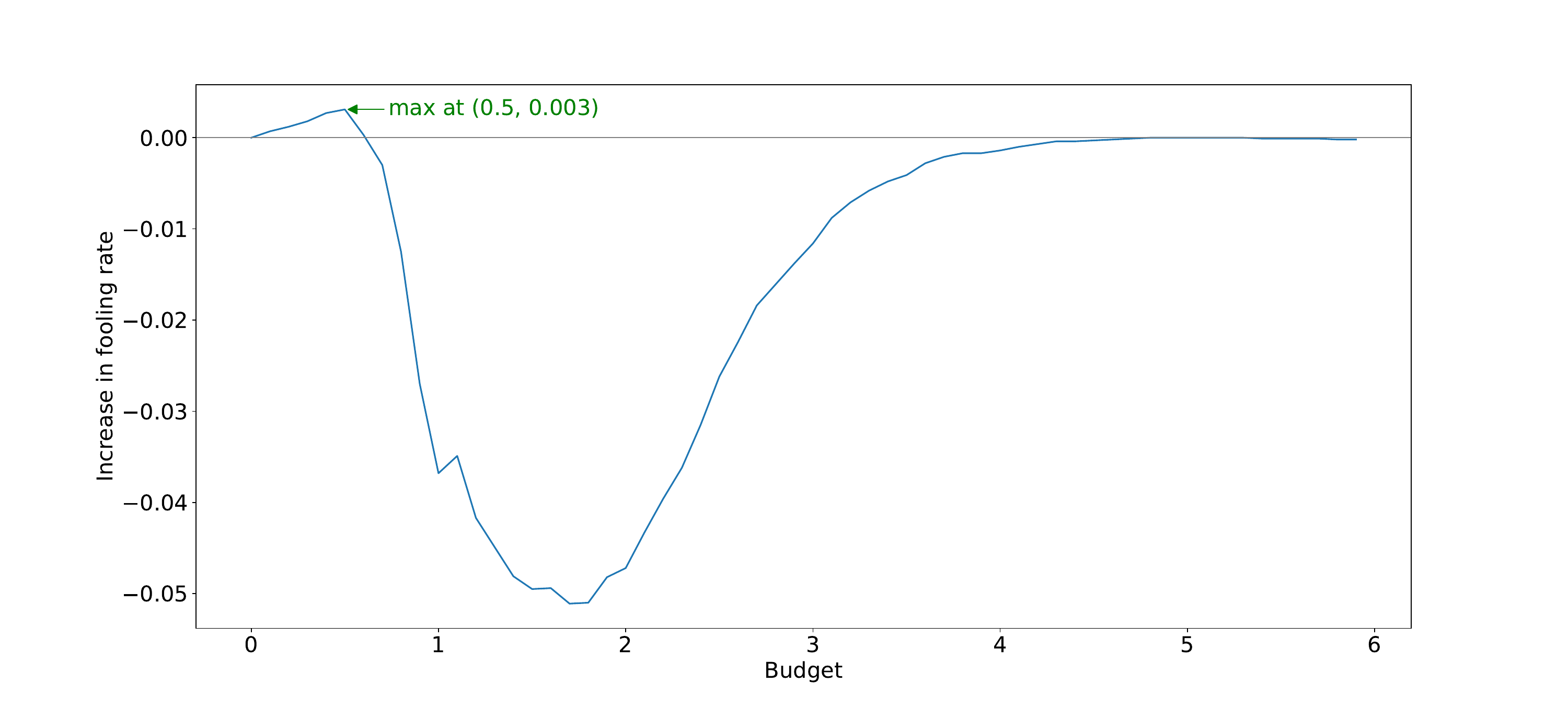}
				\caption{Activation function: ReLU}
	\end{subfigure}
	\begin{subfigure}{\textwidth}
				\centering
				\includegraphics[width=\textwidth]{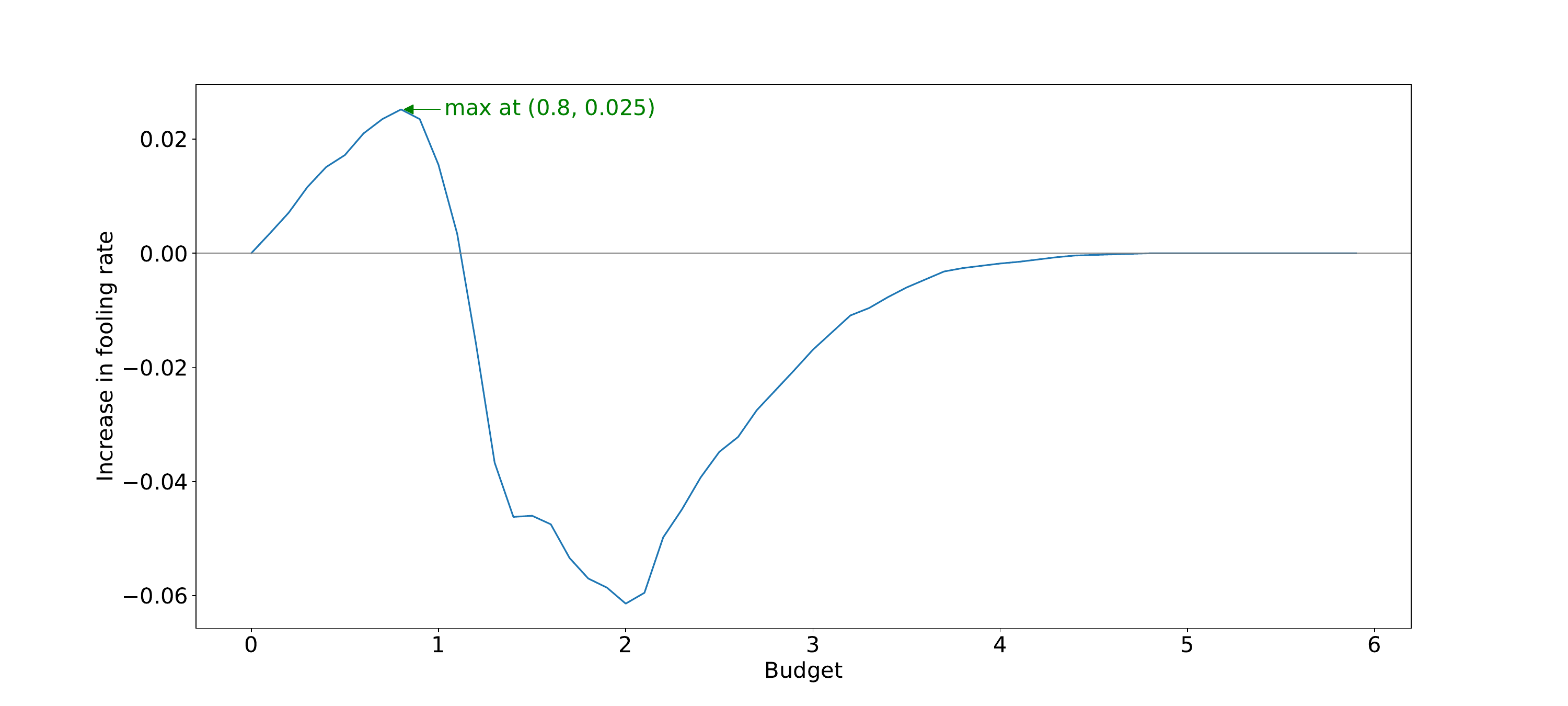}
				\caption{Activation function: Sigmoid}
	\end{subfigure}
	\caption{Difference between the fooling rate of the TSSA ($\text{fr}_{\text{TSSA}}$) and the one of AutoAttack ($\text{fr}_{\text{AA}}$) with respect to the Euclidean budget ($\text{fr}_{\text{TSSA}} - \text{fr}_\text{AA}$) on MNIST.}
	\label{fig:fr_diff_MNIST_TSSA-AA}
\end{figure}

\begin{figure}
	\centering
	\includegraphics[width=\textwidth]{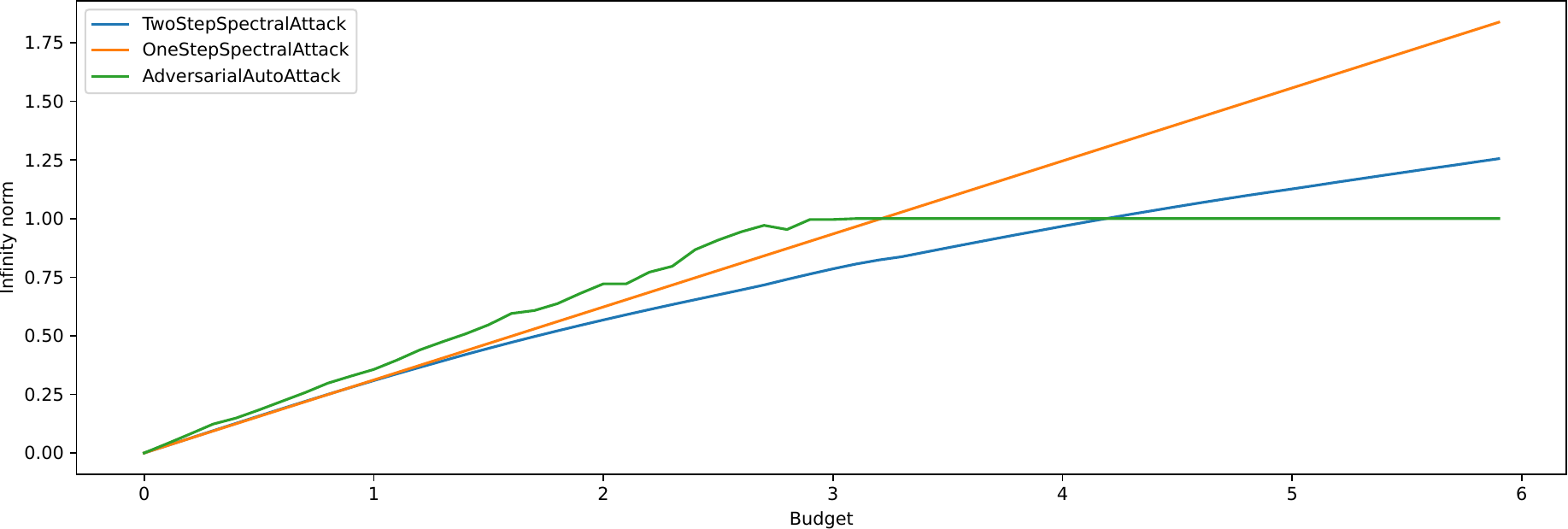}
	\caption{$\norm{\cdot}_\infty$ of the attack with respect to its Euclidean norm $\norm{\cdot}_2$ on MNIST.}
	\label{fig:MNIST-inf_norm_compared_TSSA-OSSA-AA}
\end{figure}

\begin{figure}
	\centering
	\includegraphics[width=0.7\textwidth]{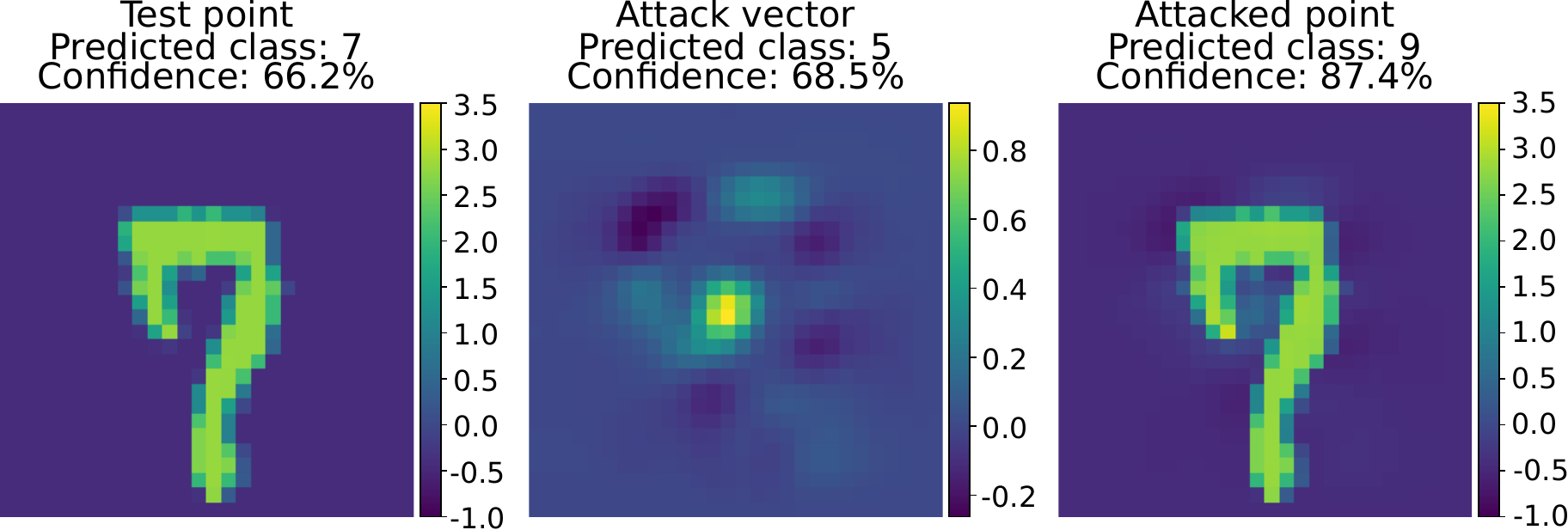}
	\caption{Result of the TSSA with a budget $\varepsilon=3$.}
	\label{fig:TSSA_bugdet_3}
\end{figure}

\begin{figure}
	\centering
	\includegraphics[width=0.7\textwidth]{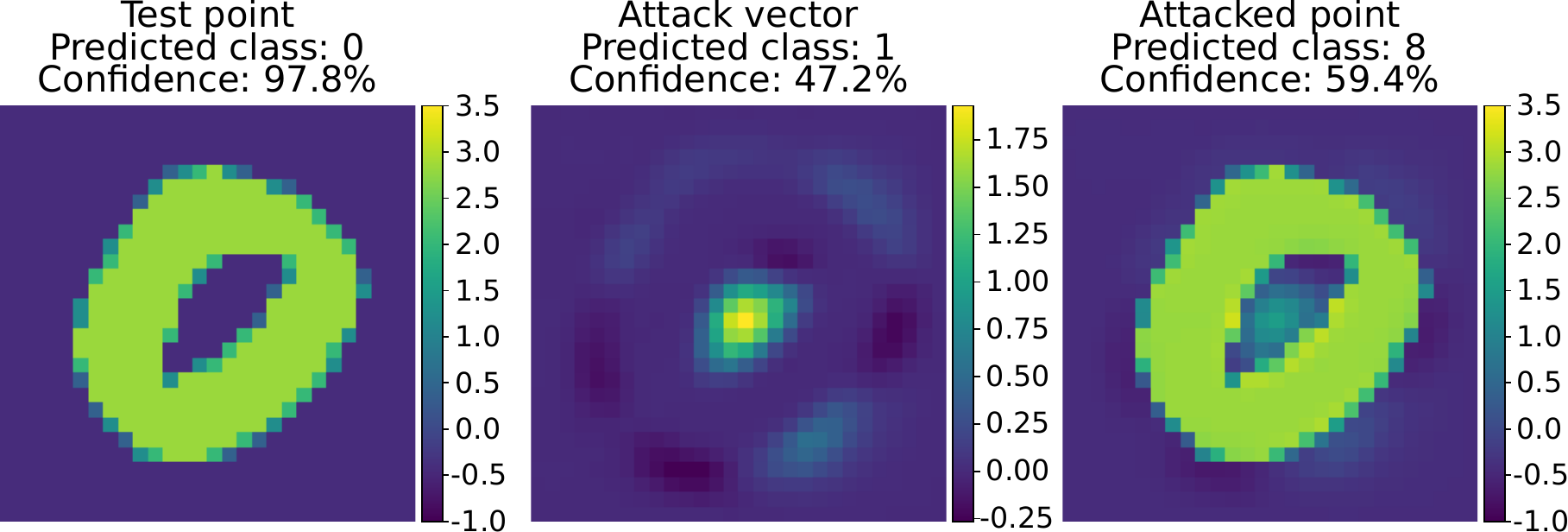}
	\caption{Result of the TSSA with a budget $\varepsilon=7$.}
	\label{fig:TSSA_bugdet_7}
\end{figure}

\minrevision{
The values of the fooling rates with respect to the Euclidean budget for the 3 examined adversarial attack procedures on MNIST are detailed in Table~\ref{tab:fr_mnist_sigmoid_table} for the Sigmoid network and in Table~\ref{tab:fr_mnist_relu_table} for the ReLU network.
\begin{table}
    \centering
    \caption{Fooling rates with respect to the Euclidean budget on all MNIST test-set (without input normalization, activation function: Sigmoid) for the 3 attack procedures.}
    \label{tab:fr_mnist_sigmoid_table}
\begin{tabular}{lccccccccccc}
\toprule
Budget ($\varepsilon$) & 0.50 & 1.00 & 1.50 & 2.00 & 2.50 & 3.00 & 3.50 & 4.00 & 4.50 & 5.00 \\
\midrule
OSSA & 0.126 & 0.326 & 0.599 & 0.805 & 0.915 & 0.963 & 0.983 & 0.992 & 0.997 & 0.998\\
TSSA & 0.128 & 0.339 & 0.643 & 0.845 & 0.945 & 0.981 & 0.994 & 0.998 & 1.000 & 1.000 \\
AA & 0.110 & 0.324 & 0.689 & 0.906 & 0.980 & 0.998 & 1.000 & 1.000 & 1.000 & 1.000\\
\bottomrule
\end{tabular}
\end{table}
\begin{table}
    \centering
    \caption{Fooling rates with respect to the Euclidean budget on all MNIST test-set (without input normalization, activation function: ReLU) for the 3 attack procedures.}
    \label{tab:fr_mnist_relu_table}
\begin{tabular}{lccccccccccc}
\toprule
Budget ($\varepsilon$) & 0.50 & 1.00 & 1.50 & 2.00 & 2.50 & 3.00 & 3.50 & 4.00 & 4.50 & 5.00\\
\midrule
OSSA & 0.088 & 0.309 & 0.614 & 0.820 & 0.923 & 0.970 & 0.988 & 0.996 & 0.998 & 0.999\\
TSSA & 0.093 & 0.349 & 0.666 & 0.862 & 0.953 & 0.985 & 0.996 & 0.999 & 1.000 & 1.000\\
AA & 0.090 & 0.386 & 0.716 & 0.909 & 0.979 & 0.997 & 1.000 & 1.000 & 1.000 & 1.000\\
\bottomrule
\end{tabular}
\end{table}

}

\revision{
\subsection{CIFAR10 dataset}
				We ran similar experiments with the CIFAR10 dataset as in Section~\ref{ssec:MNIST} with the MNIST dataset. The fooling rates for the One Step Spectral Attack and the Two Step Spectral Attack are represented on Figure~\ref{fig:fr_compared_CIFAR10} for different Euclidean budget ranging from 0 to 2. Figure~\ref{fig:fr_diff_CIFAR10_TSSA-OSSA} and Figure~\ref{fig:fr_diff_CIFAR10_TSSA-AA} show the difference between the fooling rate of the TSSA and respectively the fooling rate of OSSA and the fooling rate of AutoAttack. Figure~\ref{fig:CIFAR10-inf_norm_compared_TSSA-OSSA-AA} provides the comparison between the infinity norm of the attack of TSSA and AutoAttack with respect again to its Euclidean norm.}

\revision{Similar conclusions to the MNIST case can be drawn from the CIFAR10 results. It strengthens our confidence in the results and their interpretations stated in Section~\ref{ssec:MNIST}.
}

%revision
\begin{figure}
    \centering
		\includegraphics[width=0.7\textwidth]{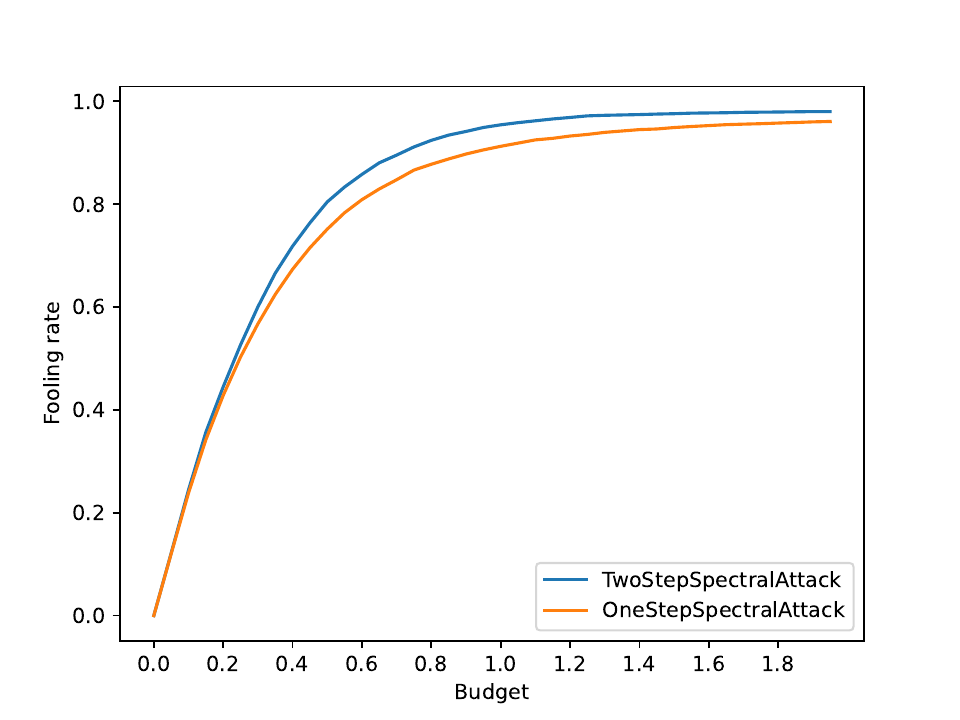}
    \caption{Fooling rates with respect to the Euclidean budget on all CIFAR10 test-set.}
    \label{fig:fr_compared_CIFAR10}
\end{figure}

\begin{figure}
    \centering
    \includegraphics[width=\textwidth]{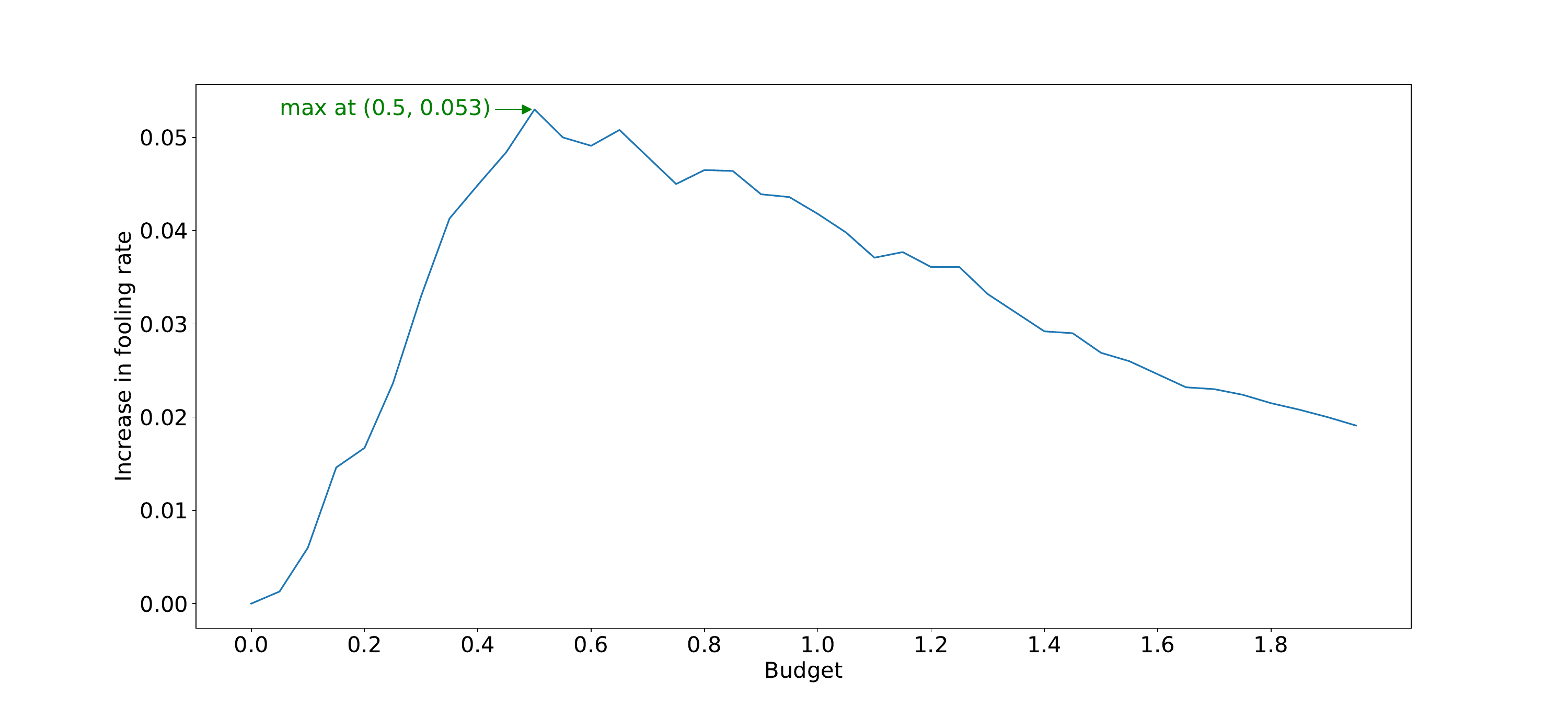}
    \caption{Difference between the fooling rate of the TSSA ($\text{fr}_{\text{TSSA}}$) and the one of the OSSA ($\text{fr}_{\text{OSSA}}$) with respect to the Euclidean budget ($\text{fr}_{\text{TSSA}} - \text{fr}_\text{OSSA}$) on CIFAR10.}
    \label{fig:fr_diff_CIFAR10_TSSA-OSSA}
\end{figure}

\begin{figure}
	\centering
				\includegraphics[width=\textwidth]{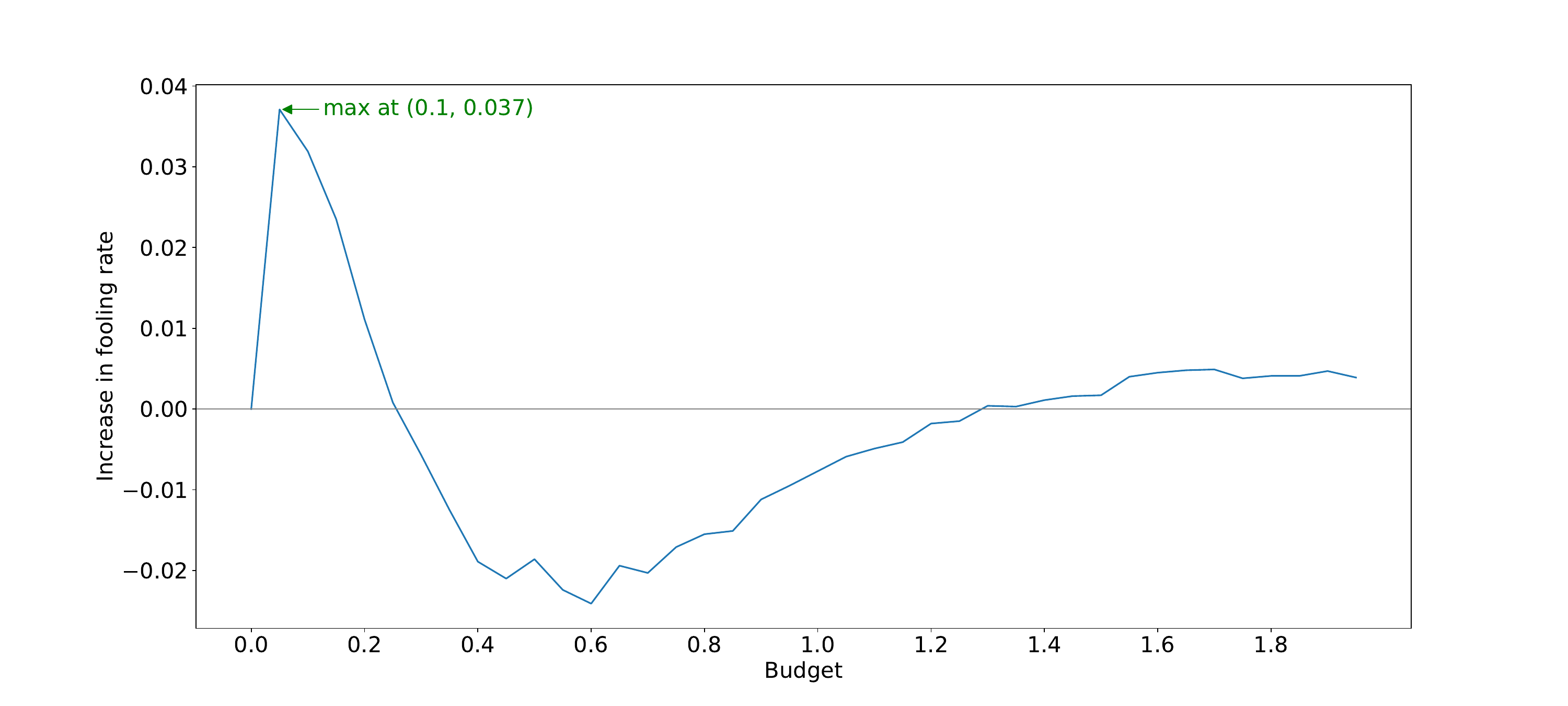}
	\caption{Difference between the fooling rate of the TSSA ($\text{fr}_{\text{TSSA}}$) and the one of AutoAttack ($\text{fr}_{\text{AA}}$) with respect to the Euclidean budget ($\text{fr}_{\text{TSSA}} - \text{fr}_\text{AA}$) on CIFAR10.}
	\label{fig:fr_diff_CIFAR10_TSSA-AA}
\end{figure}

\begin{figure}
	\centering
	\includegraphics[width=\textwidth]{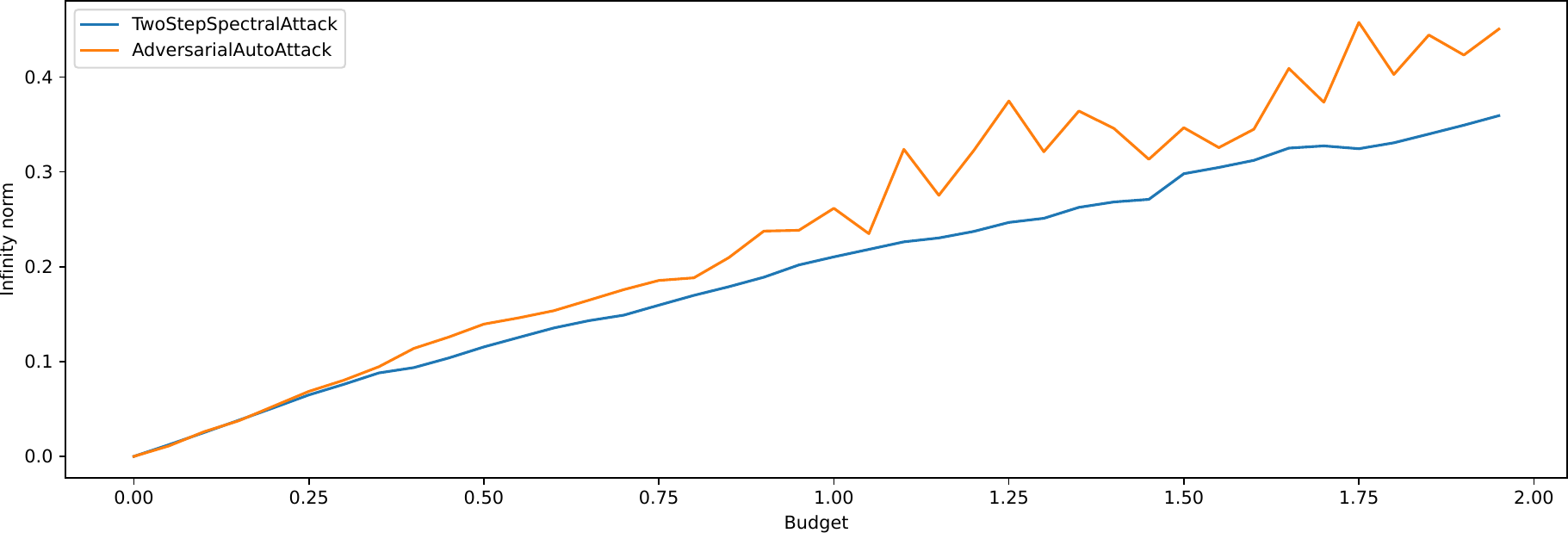}
	\caption{$\norm{\cdot}_\infty$ of the attack with respect to its Euclidean norm $\norm{\cdot}_2$ on CIFAR10.}
	\label{fig:CIFAR10-inf_norm_compared_TSSA-OSSA-AA}
\end{figure}

\minrevision{
The values of the fooling rates with respect to the Euclidean budget for the 3 examined adversarial attack procedures on CIFAR10 are detailed in Table~\ref{tab:fr_cifar10_table}.
\begin{table}
    \centering
    \caption{Fooling rates with respect to the Euclidean budget on all CIFAR10 test-set for the 3 attack procedures.}
    \label{tab:fr_cifar10_table}
\begin{tabular}{lccccccccc}
\toprule
Budget ($\varepsilon$) & 0.20 & 0.40 & 0.60 & 0.80 & 1.00 & 1.20 & 1.40 & 1.60 & 1.80\\
\midrule
OSSA & 0.428 & 0.673 & 0.809 & 0.877 & 0.912 & 0.932 & 0.945 & 0.953 & 0.958\\
TSSA & 0.445 & 0.718 & 0.858 & 0.924 & 0.954 & 0.969 & 0.974 & 0.977 & 0.979\\
AA & 0.434 & 0.737 & 0.882 & 0.939 & 0.962 & 0.970 & 0.973 & 0.973 & 0.975\\
\bottomrule
\end{tabular}

\end{table}
}

\section{Conclusion}
\label{sec:conclusion}

This paper explores the relationship between adversarial attacks and the curvature of the data space. \rerevision{Using the curvature information, we have proposed a Two Steps Attack that achieves better results than the One Step Spectral Attack presented by \cite{zhaoAdversarialAttackDetection2019} on relatively small architectures. As the computation of the attack involves the computation of the whole Jacobian matrix of the neural network, some limitations will be reached on larger architectures. The main goal of this study is the mathematical emphasis on the role of geometrical properties of the data space through the neural network Riemannian foliations rather than providing a new numerical attacking strategy that could outperform existing state-of-art attacks for various neural network architectures.} The analytical mathematical expression of the proposed attack explicitly uses the curvature tensor of the FIM kernel leaves. This emphasizes the importance of geometry in the construction of an efficient attack. Additionally, with simple experiments on a toy example, we have illustrated and confirmed that exploiting such geometrical information is relevant and actually outperforms a state of the art strategy. The mathematical construction of the proposed method opens also new opportunities for future research. Indeed, it is clear that, in the case of neural networks, the geometrical properties of the leaves of the kernel foliation is related to its robustness as explained
above but more generally to its power to separate data points. Therefore, the geometry of the foliation is directly linked to the complexity of the model (i.e. the neural network architecture). A deeper analysis of these aspects should help in gaining knowledge and some explainability of the underlying principles at play in neural network learning and more generally deep learning methods. This will be the focus of our future research.

\newpage
\begin{appendices}

\section{Notions of Riemannian geometry}
\label{app:intro_geo}
Let $\M = (M,g)$ be a real Riemannian manifold of dimension $n$.

\begin{definition}[length of a path]
Let $\gamma:\bra{0,1} \to M$ be a $C^1$ path. The \emph{length} of $\gamma$, denoted $l(\gamma)$ is defined by
\[l(\gamma) = \int_0^1 g\pa{\gamma\pa{t},\gamma'\pa{t},\gamma'\pa{t}}^{1/2} dt\]
\end{definition}

\begin{definition}[geodesic]
Let $p,~q$ be two points of $M$. The $C^1$ path $\gamma:\bra{0,1}\to M$ is said to be a \emph{geodesic} between $p$ and $q$ if:
\begin{align*}
&    \gamma(0) = p,~ \gamma(1) = q\\
&    l(\gamma) = \inf\set{l(\theta),~ \theta\in C^1\pa{[0,1),~M},~\theta(0)=p,~ \theta(1) = q
}
\end{align*}
\end{definition}

\begin{definition}[geodesic distance]
\label{def:geo_dist}
The length of a geodesic between $p$ and $q$ is called the \emph{geodesic distance}, denoted $d(p,q)$.
\end{definition}

In what follows, the Levi-Civita connexion of $\M$ will be denoted $\LC$.

\begin{definition}
Let $\gamma:\bra{0,1}\to M$ be a $C^2$ path. It it said to be a geodesic of $\LC$ is, for each $t\in]0,1[$ the following holds:
\begin{equation}
\label{eq:geodesic_LC}
\LC_{\dot{\gamma}(t)} \dot{\gamma}(t) = 0
\end{equation}
The differential equation \ref{eq:geodesic_LC} translates in local coordinates to:
\begin{equation}
\der{^2 \gamma^k}{t^2}(t) + \Gamma_{ij}^k \pa{\gamma\pa{t}}\der{\gamma^i}{t}(t) \der{\gamma^j}{t}(t) = 0
\end{equation}
\end{definition}

For an initial point $p=\gamma(0)$, Cauchy-Lipschitz theorem shows that ther is a unique local solution to \autoref{eq:geodesic_LC}.

\begin{prop}
Let $p\in M$. There exist $\varepsilon >0 $ such that for all $v\in T_pM$, $\norm{v} < \varepsilon$, there exist a unique geodesic $\gamma:\bra{0,1} \to M$ such that $\gamma(0) = p,~ \gamma'(0) = v$. The function which to such $v$ associates $\gamma(1)$ with $\gamma$ the geodesic of $\LC$ such that $\gamma(0) = p,~ \gamma'(0) =v$ is called the \emph{exponential map} and denoted $\exp_p$.
\end{prop}

\begin{prop}
Let $v\in T_pM$, $\norm{v} < \varepsilon$ and let $q = \exp_p(v)$. Then $\gamma:t\in\bra{0,1} \mapsto \exp_p(tv)$ is a geodesic between $p$ and $q$. Besides, $l(\gamma) = \norm{v}$.
\end{prop}

\begin{rem}
Note that $\norm{v} = g\pa{p; v,v}^{1/2}$ is the riemannian norm and not the Euclidean norm.
\end{rem}

\begin{definition}
\label{def:normal_coord}
$d\exp_p(0) = Id$ and thus the exponential map is a local diffeomorphism. Around each point, $\exp_p$ defines a chart of $M$. The local coordinates we get are called the \emph{normal coordinates} at $p$.
\end{definition}

\begin{definition}[logarithm map]
\label{def:log_map}
Let $p\in M$ and $\varepsilon>0$ such that the exponential map is defined in $B(0,\varepsilon)$. For all $q\in M$ such that $d(p,q) < \varepsilon$ we set:
\[\log_p(q) = v,~ \exp_p v = q.\]
\end{definition}

\begin{rem}
One can compute the logarithm by solving the following differential system:
\begin{equation}
\cas{
\der{^2 \gamma^k}{t^2}(t) + \Gamma_{ij}^k\pa{\gamma(t)} \der{\gamma^i}{t}(t) \der{\gamma^j}{t}(t) = 0 \\
\gamma(0) = p,~ \gamma(1) = q
}
\end{equation}
\end{rem}

\begin{prop}
In normal coordinates at $p\in M$, geodesics with origin $p$ are strait lines going through the origin.
\end{prop}

\begin{definition}[parallel transport]
Let $v\in T_p M$ and $q= \exp_p x$. The geodesic between $p$ and $q$ is $\gamma: t \in\bra{0,1} \mapsto tx$ in normal coordinates. Besides, the linear differential equation:
\[\LC_{\dot{\gamma}(t)}X(t) = 0,~ X(0) = v\]
has a solution on $\bra{0,1}$ called the \emph{parallel transport} of $v$.
\end{definition}

Parallel transport allows to go from a tangent vector at $q=\exp_p x$ to a tangent vector at $p$.

\section{Fisher Information Metric}
\label{app:FIM}
An important question arising when dealing with Fisher information metric is to know when going in the converse direction is feasible: given a Riemannian manifold $(\X,g)$, is it possible to find a probability family such that $g$ is exactly its Fisher information? This is exactly what is behind the next definition.
\begin{definition}[Statistical model]
\label{def:statistical_model}
A statistical model for a Riemannian manifold $(\X,g)$ is a probability space $(\Omega,\mathcal{T},P)$ such that:
\begin{itemize}
\item It exists a family of probabilities $p_x, \, x \in \X$, absolutely continuous with respect to $P$.
\item For any $x \in X$:
\[
g_{ij}(x) = E_{p_x}\left[\partial_i \ln p_x \partial_j \ln p_x\right]
\]
\end{itemize}
\end{definition}
\begin{rem}
When $p_x$ is $C^2$ with support not depending on $x$ and the conditions for exchanging derivative and expectation are satisfied, then:
\[
g_{ij}(x) = - E_{p_x}\left[\partial_{ij} \ln p_x \right]
\]
In such a case, the metric $g$ is Hessian.
\end{rem}
In nearly all cases considered in machine learning, the metric $g$ is only semi-definite. It thus makes sense to consider its kernel.
\begin{definition}
\label{def:metric_kernel}
Let $g$ be a semi-definite metric on a manifold $\X$. A tangent vector $X \in T_x\X$ is said to belong to the kernel $\ker_x g$ of $g_x$ if for any $Y \in T_x\X$, $g(X,Y)=0$. 
\end{definition}
\begin{prop}
\label{prop:integrability_kerg}
Let $(\X,g)$ be a connected manifold with $g$ a semi-definite metric. If it exists a torsionless connection  $\nabla$ on $T\X$ such that $\nabla g = 0$, then the mapping $x\in \X \to \ker_x g$ defines an integrable distribution, denoted by $\ker g$.
\end{prop}
\begin{proof}
It is clear that for any $x\in \X$, $\ker_x g$ is a linear subspace of $T_x\X$. 
Let $X,Y,Z$ be vector fields such that $Y \in \ker g $. Then, since $\nabla g = 0$ by assumption:
\[
X\left(g(Y,Z)\right) = g\left(\nabla_X Y,Z\right) + g\left(Y,\nabla_X Z\right)
\]
Since $Y \in \ker g$:
\[
X\left(g(Y,Z)\right) = 0 = g\left(\nabla_X Y,Z\right)
\]
and so, for any $X$, $\nabla_X Y \in \ker g$. This proves that the parallel transport of a vector in $\ker g$ is a vector in $\ker g$. The dimension of $\ker g$ is thus constant. Now, if $X,Y \in \ker g$, by the above result and since $\nabla$ has vanishing torsion: $[X,Y] = \nabla_X Y -\nabla_Y X \in \ker g$, proving that $\ker g$ is an integrable distribution.
\end{proof}
\begin{rem}
If the dimension of $\ker g$ is not constant, then no torsionless connection $\nabla$ can be such that $\nabla g = 0$. However, there is still a singular foliation associated with $\ker g$, with a canonical stratification by the dimension of $\ker g$.
\end{rem}
\begin{prop}
\label{prop:transverse_metric}
Under the assumptions of prop. \ref{prop:integrability_kerg}, $g$ defines a transverse metric for the $\ker g$ foliation.
\end{prop}
\begin{proof}
This is essentially prop 3.2, p. 78 in \cite{1988riemannian}.
\end{proof}
The leaves of the $\ker g$ foliation are neutral submanifolds for the fisher information metric, that is moving along them will not modify the output distribution. On the other hand, the transverse metric is a measure of output variation when moving in a direction normal to the leaves. 

\revision{
\section{Proofs}
\label{app:proofs}

\begin{proof}[Proof of Proposition~\ref{prop:ossa_vector}]

To maximize this quadratic form under the constraint $\norm{v}_2^2 = \varepsilon^2$, Karush-Kuhn-Tucker's optimality conditions\footnote{See \cite{kuhn2014nonlinear} for more details.} ensure that there exists a scalar $\lambda \in\R$ such that the optimum attack $\hat{v}$ satisfies:
\begin{align*}
				& \nabla_v \pa{v^T G_x v} - \lambda \nabla_v\pa{\norm{v}_2^2 - \varepsilon^2} = 0 \\
				\implies & \pa{G_x + G_x^T}{\hat{v}} = \lambda 2 {\hat{v}} \\
				\implies & G_x {\hat{v}} = \lambda {\hat{v}}
.\end{align*}

Otherwise said, the optimum ${\hat{v}}$ is in the set of $G_x$'s eigenvectors. In this case:

\begin{equation*}
				\norm{{\hat{v}}}_\X^2 = {\hat{v}}^T G_x {\hat{v}}
										= {\hat{v}}^T \lambda {\hat{v}}
										= \lambda \norm{{\hat{v}}}_2^2
										= \lambda \varepsilon^2
.\end{equation*}

Thus, ${\hat{v}}$ corresponding to the largest eigenvalue maximizes the Riemannian norm.

\end{proof}
}
\end{appendices}

\bibliography{main}
\end{document}